%% file: icml.tex
\documentclass[nohyperref]{article}

\usepackage{microtype}
\usepackage{graphicx}
\usepackage{subfigure}
\usepackage{booktabs} %

\usepackage{hyperref}

 \usepackage[accepted]{icml2022}

\usepackage{amsmath}
\usepackage{amssymb}
\usepackage{mathtools}
\usepackage{amsthm}

\usepackage[capitalize,noabbrev]{cleveref}

\input{preamble.tex}

\usepackage[textsize=tiny]{todonotes}

\icmltitlerunning{Algorithms for hidden stratification and multi-group learning}

\begin{document}

\twocolumn[
\icmltitle{Simple and near-optimal algorithms \\ for hidden stratification and multi-group learning}

\icmlsetsymbol{equal}{*}

\begin{icmlauthorlist}
\icmlauthor{Christopher Tosh}{msk}
\icmlauthor{Daniel Hsu}{columbia}
\end{icmlauthorlist}

\icmlaffiliation{msk}{Memorial Sloan Kettering Cancer Center, New York, NY}
\icmlaffiliation{columbia}{Department of Computer Science, Columbia University, New York, NY}

\icmlcorrespondingauthor{Christopher Tosh}{christopher.j.tosh@gmail.com}
\icmlcorrespondingauthor{Daniel Hsu}{djhsu@cs.columbia.edu}

\icmlkeywords{multi-group learning, fairness, hidden stratification}

\vskip 0.3in
]

\printAffiliationsAndNotice{}  %

\begin{abstract}
\input{abstract.tex}

\end{abstract}

\input{main.tex}

\section*{Acknowledgements}
We thank Kamalika Chaudhuri for helpful initial discussions about hidden
stratification.
We acknowledge support from NSF grants CCF-1740833 and IIS-1563785, and a JP
Morgan Faculty Award.
Part of this work was completed while CT was at Columbia University.

\bibliography{../refs}
\bibliographystyle{icml2022}

\newpage
\appendix
\onecolumn

\input{appendix.tex}

\end{document}

%% file: preamble.tex
\usepackage{mkolar_definitions}

\usepackage[utf8]{inputenc} %
\usepackage[T1]{fontenc}    %
\usepackage{url}            %
\usepackage{booktabs}       %
\usepackage{amsfonts}       %
\usepackage{nicefrac}       %
\usepackage{microtype}      %
\usepackage{xspace}
\usepackage{amsmath}
\usepackage{color}
\usepackage{enumitem}
\usepackage{comment}
\usepackage{bm}
\usepackage{graphicx}
\usepackage[many]{tcolorbox}

\usepackage{apptools}
\usepackage[page, header]{appendix}
\usepackage{titletoc}

\usepackage{mdframed}
\mdfsetup{skipabove=0pt,skipbelow=0pt}
\usepackage[scaled=.9]{helvet}

\usepackage{url}
\usepackage{authblk}
\usepackage{caption}

\newcommand{\poly}{\operatorname{poly}}

\newcommand{\sign}{\operatorname{sign}}
\newcommand{\prepend}{\textsc{Prepend}\xspace}
\newcommand{\mlchedge}{\textsc{MLC-Hedge}\xspace}

\renewcommand\algorithmicrequire{\textbf{input}}
\renewcommand\algorithmicensure{\textbf{output}}

\def\DL{\texttt{DL}}

\makeatletter
\newtheorem*{rep@theorem}{\rep@title}
\newcommand{\newreptheorem}[2]{%
\newenvironment{rep#1}[1]{%
 \def\rep@title{#2 \ref*{##1}}%
 \begin{rep@theorem}}%
 {\end{rep@theorem}}}
\makeatother

\newreptheorem{theorem}{Theorem}
\newreptheorem{lemma}{Lemma}

\newenvironment{myquote}[1]%
  {\list{}{\leftmargin=#1\rightmargin=#1}\item[]}%
  {\endlist}

\def\thresh{\textnormal{thresh}}
\def\pr{\textnormal{Pr}}

%% file: abstract.tex
Multi-group agnostic learning is a formal learning criterion that is concerned
with the conditional risks of predictors within subgroups of a population.
The criterion addresses recent practical concerns such as subgroup fairness and
hidden stratification.
This paper studies the structure of solutions to the multi-group learning
problem, and provides simple and near-optimal algorithms for the learning
problem.

%% file: main.tex
\section{Introduction}

Despite its status as the de facto selection criterion for machine learning models, accuracy is an aggregate statistic that often obscures the underlying structure of mistaken predictions.
\citet{oakden2020hidden} recently raised this concern in the context of medical image analysis.
Consider the problem of diagnosing an image as cancerous or not.
Certain types of aggressive cancers may be less common than some non-aggressive types. Thus, classifiers that concentrate their errors on images of these rarer and more aggressive cancers can achieve higher overall accuracy than classifiers that spread their errors more evenly over all types of cancer.
However, choosing classifiers that concentrate their errors on these rarer cancers is clearly not ideal, as it could lead to harmful misdiagnoses for those who need treatment the most. \citeauthor{oakden2020hidden} refer to this general phenomenon as the \emph{hidden stratification} problem, where there is some latent grouping of the data domain and performance on these latent groups is just as important as performance on the entire domain.

Similar scenarios arise in areas where \emph{fairness} is a concern~\citep[see, e.g.,][]{hardt2016equality}. Here, the concern is that for applications such as credit recommendation or loan approval the errors of models may be concentrated on certain demographic groups and potentially exacerbate pre-existing social disadvantages. This issue can persist even when protected class information such as race or age is not explicitly included in the model, as other features often serve as good proxies for protected class information.

The \emph{multi-group (agnostic) learning} setting, formalized by \citet{rothblum2021multi}, is a learning-theoretic model for addressing these scenarios. 
This setting is specified by a collection of groups $\Gcal$, where each group $g \in \Gcal$ is a subset of the input space, and a set of reference predictors $\Hcal$.
Here, the groups in $\Gcal$ can overlap in arbitrary ways, and $|\Gcal|$ need not be finite.
The multi-group learning objective is to find a predictor $f$ such that, for all groups $g \in \Gcal$ simultaneously, the average loss of $f$ among examples in $g$ is comparable with that of the best predictor $h_g \in \Hcal$ specific to $g$.
This learning objective thus pays attention to the group that is worst off with respect to $\Hcal$.
Note that because a good reference predictor $h_g$ for one group $g$ may be very poor for another group $g'$, a successful multi-group learner may need to choose its predictor $f$ from outside of $\Hcal$.

\citet{rothblum2021multi} obtained initial results for multi-group learning (and, in fact, study a broader class of objectives compared to what we consider here), but they leave open several theoretical questions that we address in this paper.
First, while it is known that uniform convergence of empirical risks with
respect to $\Hcal$ is necessary and sufficient for the standard agnostic
learning model, it is not clear whether the same holds for multi-group
learning, let alone whether the sample complexities are equivalent.
Second, \citeauthor{rothblum2021multi} focus on finite $\Gcal$, but the
motivation from hidden stratification may require the use of rich and infinite
families of groups in order to well-approximate the potentially unknown strata
of importance.
Finally, \citeauthor{rothblum2021multi} obtained their algorithm via a
blackbox reduction to a more general problem, leaving open the
possibility of simpler algorithms and prediction rules with multi-group learning guarantees.

\subsection{Summary of results}

We introduce some notation in order to state our results.
Given $n$ i.i.d.\ training examples drawn from a distribution $\Dcal$ over $\Xcal \times \Ycal$, let $\#_n(g)$ for a group $g \subseteq \Xcal$ denote the number of training examples $(x,y)$ with $x \in g$.
For a predictor $f \colon \Xcal \to \Zcal$ and a group $g \subseteq \Xcal$, let 
\[L(f \mid g) = \EE_{(x,y) \sim \Dcal}[\ell(f(x),y) \mid x \in g]\]
denote the \emph{conditional risk of $f$ on $g$}, and $\ell \colon \Zcal \times \Ycal \to [0,1]$ is a bounded loss function.

Our first multi-group learning result is an algorithm to learn simple predictors with per-group conditional risk guarantees. Here, the class of `simple predictors' we consider is the collection of decision lists in which internal (decision) nodes are associated with membership tests for groups from $\Gcal$, and leaf (prediction) nodes are associated with reference predictors from $\Hcal$.

\begin{reptheorem}{thm:large-strata-complexity}[Informal]
There is an algorithm $\Acal$ such that the following holds for any hypothesis set $\Hcal$ and set of groups $\Gcal$. Given $n$ i.i.d.\ training examples from $\Dcal$, $\Acal$ produces a decision list $f$ such that, with high probability,
\[ L(f \mid g) \ \leq \ \inf_{h \in \Hcal} L(h \mid g) + O \left(  \left( \frac{\log |\Hcal| |\Gcal|}{\gamma_n \cdot \#_n(g)} \right)^{1/3} \right) \]
for all $g \in \Gcal$, where $\gamma_n := \min_{g \in \Gcal} \#_n(g)/n$ is the minimum empirical probability mass among groups in $\Gcal$.\footnote{Here and in the rest of the paper, big-$O$ notation is only used to conceal constants. So, $a=O(b)$ should be read as `There exists an absolute constant $C > 0$ such that $a \leq Cb$.'}
\end{reptheorem}

When $\Hcal$ and $\Gcal$ are infinite, a version of \pref{thm:large-strata-complexity} also holds when the complexities of $\Hcal$ and $\Gcal$ are appropriately bounded; see \pref{thm:infinite-strata-complexity} for a formal statement.

Though the algorithm and resulting predictors of \pref{thm:large-strata-complexity} and \pref{thm:infinite-strata-complexity} are quite simple, the per-group excess error rates are suboptimal. Statistical learning theory suggests that if we knew \emph{a priori} which group $g$ we would be tested on, 
empirical risk minimization (ERM) on the i.i.d.\ training examples restricted to $g$ 
would lead to an excess risk of $O( \sqrt{\log (|\Hcal|)/ \#_n(g)})$ in the finite setting~\citep{shalev2014understanding}. 
The rates in \pref{thm:large-strata-complexity} have two undesirable properties when compared to this theoretical rate: they have a worse exponent, and they depend on the minimum probability mass among all the groups.

Our next result shows that the theoretical rate suggested by per-group ERM can be achieved in the multi-group learning setting, modulo a logarithmic factor in $|\Gcal|$.

\begin{reptheorem}{thm:sleeping-experts}[Informal]
There is an algorithm $\Acal$ such that the following holds for any finite hypothesis set $\Hcal$ and finite set of groups $\Gcal$. Given $n$ i.i.d.\ training examples from $\Dcal$, $\Acal$ produces a randomized predictor $f$ such that, with high probability,
\[ L(f \mid g) \ \leq \ \inf_{h \in \Hcal} L(h \mid g) + O \left(  \left( \frac{\log(|\Hcal| |\Gcal|)}{\#_n(g)} \right)^{1/2} \right)  \]
for all $g \in \Gcal$.
\end{reptheorem}

The improved rates of \pref{thm:sleeping-experts} come at the expense of increased complexity in both learning procedure and prediction algorithm. %

Finally, we note that in the \emph{group-realizable} setting, where each group has a corresponding perfect predictor, the rates can be further improved (specificaly, the square-root in \pref{thm:sleeping-experts} can be removed).

\subsection{Related work}

There are a number of areas of active research that intersect with the present paper.

\paragraph{Distributionally robust optimization.} The field of distributionally robust optimization, or DRO, focuses on the problem of optimizing some cost function such that the cost of the solution is robust to perturbations in the problem instance~\citep{bental2009robust}. In the context of machine learning and statistics, the idea is to use data from some training distribution to learn a predictor that will perform well when used on a worst-case choice of test distribution from some known class. In this field, the class of test distributions is typically a small perturbation of the training distribution~\citep{bertsimas2018data, duchi2021statistics, duchi2020distributionally}.

A special case of DRO of particular relevance to the current work is group-wise DRO. Here, the class of test distributions is fixed to be a finite set of distributions, and the goal is to find a predictor whose worst-case conditional risk over any of these test distributions is minimized. Several solutions for this problem have been proposed, including mirror descent~\citep{mohri2019agnostic}, group-wise regularization~\citep{sagawa2020distributionally}, group-wise sample reweighting~\citep{sagawa2020investigation}, and two-stage training procedures~\citep{liu2021just}. 

\paragraph{Group and individual fairness.} As discussed above, prediction mistakes committed by machine-learned models can lead to widespread social harms, particularly if they are concentrated on disadvantaged social groups. To address this, a recent, but large, body of work has emerged to define fairness criteria of predictors and develop machine learning methods that meet these criteria. While many criteria for fair predictors abound, they can typically be broken into two categories. The first of these is group-wise fairness~\citep{hardt2016equality, agarwal2018reductions, donini2018empirical} in which a classifier is trained to equalize some notion of harm, such as false-negative predictions, or benefit, such as true-positive predictions, across predefined groups. The second category of fairness notion is individual fairness~\citep{dwork2012fairness, dwork2018individual}, in which there is a distance function or notion of similarity among points, and the objective is to give similar predictions for similar points.

Of particular relevance to the present paper is the work of \citet{kearns2018preventing}, which studied group fairness for a large (potentially infinite) number of (potentially overlapping) groups. The authors assumed boundedness of the VC-dimensions for both the hypothesis class and the class of groups, as well as access to an oracle for solving certain cost-sensitive classification problems. Under these assumptions, they provided an algorithm that solves for a convex combination of hypotheses from the hypothesis class that respects certain notions of fairness for all groups and is competitive with the best such fair convex combination.

\paragraph{Online fairness.} There is a growing body of work at the intersection of online learning and fairness~\citep{gillen2018online, noarov2021online, gupta2021online}. The most relevant to the present paper is the work of \citet{blum2020advancing}, which studies an online version of multi-group learning where the goal is to achieve low regret on each of the groups simultaneously. That work gives a reduction to sleeping experts, showing that for a particular choice of experts, the regret guarantees of the sleeping experts algorithm directly translates to a per-group regret guarantee. Inspired by this observation, in \pref{sec:sleeping-experts} we show that the offline multi-group learning problem can also be reduced to the sleeping experts problem. The online-to-batch conversion argument we use requires some care, however, since there are multiple objectives (one for each group) that need to be satisfied, in contrast to standard online-to-batch settings where only a single objective needs to be met.

\paragraph{Multicalibration, multiaccuracy, and outcome distinguishability.} %
Also motivated by fairness considerations, a recent line of work is centered on {calibrating} predictions to be unbiased on a collection of groups or subpopulations. 

Given a class of groups $\Gcal$, \emph{multiaccuracy} requires that the expectation of a predictor is close to expectation of the outcome $y$ when conditioned on any $g \in \Gcal$~\citep{hebert2018multicalibration, diana2021multiaccurate}. \citet{kim2019multiaccuracy} showed that for an appropriate choice of groups, multiaccuracy implies a type of multi-group learnability result. Unfortunately, the upper bound they show for the per-group error rate is only non-trivial when the best rate achievable at that group is small. In \pref{app:insuf-multiaccuracy}, we show that this looseness is inherent to this type of multiaccuracy reduction.

\emph{Multicalibration} is a more stringent notion than multiaccuracy that requires these expectations to be close when conditioned both on the group and the value of the prediction. Even more stringent is \emph{outcome indistinguishability}, a family of criteria that requires the predictions of a function $f: \Xcal \rightarrow [0,1]$ to be indistinguishable against the true probabilities of positive versus negative outcomes with respect to classes of distinguishing algorithms with varying levels of access to the underlying distribution~\citep{dwork2021outcome}. 

Using a reduction to the outcome indistinguishability framework, \citet{rothblum2021multi} provided an algorithm with a multi-group learning guarantee (under a certain ``multi-PAC compatibility'' assumption). Specifically, they showed that for a given finite hypothesis class $\Hcal$ and finite set of groups $\Gcal$, one can produce a predictor $f:\Xcal \rightarrow \RR$ such that
$L(f \mid g) \ \leq \ \min_{h \in \Hcal} L( h \mid g) + \epsilon$
for all $g \in \Gcal$ with probability $1-\delta$.\footnote{In fact, their results apply to more general types of objectives that need not be \emph{decomposable}; see \citep[Section 2]{rothblum2021multi} for details. In this paper, we focus only on objectives of the form $L(f \mid g)$, which are decomposable in their sense.} The sample complexity of their approach is $O\left(\frac{m_{\Hcal}(\epsilon, \delta)^4}{\delta^4 \gamma} \log \frac{|\Hcal| |\Gcal|}{\epsilon} \right)$, where $m_{\Hcal}(\epsilon, \delta)$ is the sample complexity of agnostically learning a classifier $h \in \Hcal$ with excess error $\epsilon$ and failure probability $\delta$~\citep{rothblum2021personal}. Standard ERM arguments give us $m_{\Hcal}(\epsilon, \delta) = O \left(\frac{1}{\epsilon^2} \log \frac{|\Hcal|}{\delta} \right)$, leading to an overall sample complexity of $O\left( \frac{1}{\epsilon^8 \delta^4 \gamma} \poly \log \frac{|\Hcal||\Gcal|}{\epsilon\delta} \right)$.

In independent and concurrent work, \citet{globus2022beyond} also considered a multi-group learning setup under the framework of `bias bug bounties.' In this setting, there is some deployed model and outsiders are incentivized to find groups on which the model does worse than Bayes optimal. If such a group is found, one can submit the group as well as a certificate of suboptimality to receive a bounty, and the model will be updated to improve performance on the affected group. Interestingly, the algorithm developed by \citeauthor{globus2022beyond} to update their model (which they call \textsc{ListUpdate}) is equivalent to one of the algorithms presented in the present paper (\prepend). Beyond the development of the concept of bias bug bounties, \citeauthor{globus2022beyond} are primarily concerned with the computational complexity of finding such bounties, whereas the present work is focused on the statistical sample complexity of multi-group learning.

\paragraph{Hidden stratification.} The \emph{hidden stratification} problem refers to settings where there are meaningful subgroups or divisions of the data space, but they are unknown ahead of time. The fear, as illustrated by \citet{oakden2020hidden}, is that prediction errors can be concentrated on some important subgroup and lead to real-world harms. \citet{sohoni2020no} proposed addressing this problem by clustering a dataset and solving a DRO problem as if the resulting clusters were the true known subgroups.

\paragraph{Transfer learning and covariate shift.} Broadly speaking, \emph{transfer learning} studies the problem of learning a predictor given many samples from a source distribution $P$ and relatively few  (or perhaps no) samples from a target distribution $Q$, where the predictor will ultimately be evaluated on $Q$. The results in this area depend on what is allowed to change between $P$ and $Q$. Some works study the setting of \emph{covariate shift}, where only the covariate distribution is allowed to change~\citep{shimodaira2000improving, zadrozny2004learning, cortes2010learning, kpotufe2018marginal, hopkins2021realizable}. Others focus on \emph{label shift}, where only the marginal label distribution changes, leaving the class conditional distributions unchanged~\citep{azizzadenesheli2018regularized}. Finally, in the most general setting, $P$ and $Q$ may differ in both covariates and labels~\citep{ben2010theory}.

Of these transfer learning settings, the covariate shift framework most closely resembles the multi-group learning setting. However, the two key differences are that (1) the transfer learning setting is typically concerned with performance on a single target distribution and (2) the target distributions that arise in multi-group learning are restricted to conditional distributions of the source distribution. Importantly, in the multi-group learning setup, the target distributions never have support that is outside of the source distribution.
Because of this restriction, we can achieve a much stronger performance guarantee compared to what is possible in the setting of general covariate shift.

One work that deviates from the single-target distribution framework is that of \citet{hopkins2021realizable}, whose notion of covariate shift error is the maximum error over a set of target distributions, similar to the multi-group learning setup.
However, in the covariate shift setting of \citet{hopkins2021realizable}, there is only a single such benchmark hypothesis for all possible shifts, whereas in our multi-group setting, the benchmark hypothesis is allowed to depend on the group.
Thus the guarantees in the setting of \citet{hopkins2021realizable} are not comparable to those in our setting.

\paragraph{Boosting.} \emph{Boosting} is a classical machine learning technique for converting weak learners, i.e., learners that output predictors that are only marginally more accurate than random guessing, into strong learners, i.e., learners that output predictors with very high accuracy~\citep{schapire1990strength, freund1995boosting}. Many of the algorithms for achieving multiaccuracy, multicalibration, and outcome indistinguishability can be viewed as boosting algorithms~\citep{hebert2018multicalibration, kim2019multiaccuracy, dwork2021outcome}.
For instance, the algorithm of \citet{kim2019multiaccuracy} is based on the boosting algorithm of \citet{trevisan2009regularity}.
One of the algorithms proposed in this paper, \prepend, is no exception here and may also be viewed as a boosting algorithm.

\subsection{Paper outline}

The remainder of the paper is organized as follows. In \pref{sec:setting}, we formalize the multi-group learning setting. In \pref{sec:simple-algorithm} we present a simple algorithm for the multi-group learning problem that achieves a suboptimal generalization bound. In \pref{sec:sleeping-experts}, we give a reduction to the online sleeping experts problem. The resulting algorithm achieves the correct generalization rate, though this comes at the expense of a significantly more complicated learning algorithm. Finally, in \pref{sec:realizable}, we consider a setting in which each group has a corresponding perfect classifier.

All proofs are presented in the appendix.

\section{Setting and notation}
\label{sec:setting}
Let $\Xcal$ denote an input space, $\Ycal$ denote a label space, and $\Zcal$ denote a prediction space.
Let $\Dcal$ denote a distribution over $\Xcal \times \Ycal$.
Throughout, $\Hcal \subseteq \{h: \Xcal \rightarrow \Zcal \}$ denotes a (benchmark) hypothesis class. A group $g$ is a subset of the space $\Xcal$. We overload notation by identifying a group $g \subseteq \Xcal$ with the binary function $g: \Xcal \rightarrow \{ 0, 1 \}$ that indicates membership in $g$. We denote the set of groups of interest by $\Gcal$, and let $P(g) := \EE_{(x,y) \sim \Dcal}[ g(x) ]$ for any group $g$.
Let $\ell \colon \Zcal \times \Ycal \to [0,1]$ be a bounded loss function, and for a predictor $f \colon \Xcal \to \Zcal$, the conditional risk of $f$ given $g$ is \[L(f \mid g) := \EE_{(x,y) \sim \Dcal}[ \ell(f(x),y) \mid x \in g ].\]
For an i.i.d.~sample $(x_1,y_1),\dotsc,(x_n,y_n) \sim \Dcal$, we define the following empirical quantities: let $\#_n(g) := \sum_{i=1}^n g(x_i)$, $P_n(g) := \#_n(g) / n$, and \[L_n(f \mid g) := \frac1{\#_n(g)} \sum_{i=1}^n g(x_i) \ell(f(x_i),y_i).\]
The (unconditional) risk and empirical risk of $f$ are $L(f) := \EE_{(x,y) \sim \Dcal}[\ell(f(x),y)]$ and $L_n(f) := \frac1n \sum_{i=1}^n \ell(f(x_i),y_i)$.

\subsection{Multi-group agnostic learning}

At a high level, the objective in the multi-group (agnostic) learning setting is to find a predictor $f$ such that the conditional risk is not much larger than $\inf_{h \in \Hcal} L(h \mid g)$ for all $g \in \Gcal$.
This setting was formalized by \citet{rothblum2021multi}; they require, for a
given $\epsilon>0$, that the excess conditional risks be uniformly small over
all groups $g \in \Gcal$:
\begin{align}
\label{eqn:hs_obj_abs}
L(f \mid g) \ \leq \ \inf_{h \in \Hcal} L(h \mid g) + \epsilon
\quad \text{for all $g \in \Gcal$} .
\end{align}
Note that the best hypothesis in the benchmark class $h_g \in \Hcal$ for a particular group $g$ may not be the same as that for a different group $g'$.
Indeed, there may be no single $h \in \Hcal$ that has low conditional risk for all groups in $\Gcal$ simultaneously.
Hence, a learner may typically need to choose a predictor $f$ from outside of $\Hcal$.

We will give non-uniform bounds on the excess conditional risks (discussed below), where $\epsilon$ is replaced by a quantity $\epsilon_n(g)$ depending on both the size $n$ training set and the specific group $g \in \Gcal$ in question:
\begin{align}
\label{eqn:hs_obj_rel}
L(f \mid g) \ \leq \ \inf_{h \in \Hcal} L(h \mid g) + \epsilon_n(g) \quad \text{for all $g \in \Gcal$} .
\end{align}
The quantity $\epsilon_n(g)$ will be a decreasing function of $\#_n(g)$, and it will be straightforward to determine a minimum sample size $n$ (in terms of a prescribed $\epsilon>0$) such that Eq.~\eqref{eqn:hs_obj_rel} implies Eq.~\eqref{eqn:hs_obj_abs}.

\subsection{Convergence of conditional risks}

For a class of $\{0,1\}$-valued functions $\Fcal$ defined over a domain $\Xcal$, the $k$-th shattering coefficient, is given by
\[ \Pi_k(\Fcal) := \max_{x_1, \ldots, x_k \in \Xcal} \left| \{ (f(x_1), \ldots, f(x_k)) \, : \, f \in \Fcal \} \right| . \]
For a class of real-valued functions $\Fcal$ defined over a domain $\Xcal$, the thresholded class is given by
\[ \Fcal_{\thresh} \ := \ \{ x \mapsto \ind[f(x) > \tau ] \, : \, f \in \Fcal, \tau \in \RR \}. \]
Finally, for a hypothesis class $\Hcal$ and a loss function $\ell \colon \Zcal \times \Ycal \to [0,1]$, the loss-composed class is 
\[ \ell \circ \Hcal \ := \ \{  (x, y) \mapsto \ell(h(x), y) \, : \, h \in \Hcal \}. \]
The following theorem shows that the empirical conditional risks converge uniformly to their population counterparts. This can be seen as a generalization of a result by \citet{BDFM19}, which demonstrated universal convergence of empirical conditional probabilities. 

\begin{theorem}
\label{thm:conv_conditional_losses}
Let $\Hcal$ be a hypothesis class, let $\Gcal$ be a set of groups, and let $\ell \colon \Zcal \times \Ycal \to [0,1]$ be a loss function. With probability at least $1-\delta$,
\[ |L(h \mid g) - L_n(h \mid g )| \ \leq \ 9 \sqrt{\frac{D}{\#_n(g)}} \quad \forall (h,g) \in \Hcal \times \Gcal,  \]
where $D = 2\log \left(\Pi_{2n}((\ell \circ \Hcal)_{\thresh}) \Pi_{2n}(\Gcal) \right) + \log(8/\delta)$.
\end{theorem}

In the standard agnostic binary classification setting, it is known that, in general, the best achievable error rate of a learning algorithm is on the order of the uniform convergence rate of the empirical risks of the entire hypothesis class~\citep[Chapter 6]{shalev2014understanding}. This can be seen as a statistical equivalence between learning and estimation. \pref{thm:conv_conditional_losses} raises the question of whether such an equivalence can also be established in the multi-group learning setting. In this work, we make partial progress towards establishing such an equivalence, providing a learning algorithm whose per-group error rate enjoys the same upper bound as the convergence rates in \pref{thm:conv_conditional_losses}.

\section{A simple multi-group learning algorithm}
\label{sec:simple-algorithm}

In this section we will show that there is a particularly simple class of predictors for solving
the multi-group learning problem: decision lists in which internal (decision)
nodes are associated with functions from $\Gcal$, and leaf (prediction) nodes
are associated with functions from $\Hcal$.
We denote the set of such decision lists of length $t$ by $\DL_t[\Gcal; \Hcal]$.
The function computed by $f_t = [g_t, h_t, g_{t-1}, h_{t-1}, \ldots, g_1, h_1, h_0] \in \DL_t[\Gcal; \Hcal]$ is as follows: upon input $x \in \Xcal$,
\begin{myquote}{0.2in}
  \textbf{if} $g_t(x) = 1$ \textbf{then} return $h_t(x)$ \textbf{else if} $g_{t-1}(x) = 1$ \textbf{then} return $h_{t-1}(x)$ \textbf{else if} \ $\cdots$ \ \textbf{else} return $h_0(x)$.
\end{myquote}
This computation can be recursively specified as
\[ f_t(x) \ = \ \begin{cases}
h_t(x) & \text{if $g_t(x) = 1$} \\
f_{t-1}(x) & \text{if $g_t(x) = 0$}
\end{cases} \]
where $f_{t-1} =  [g_{t-1}, h_{t-1}, \ldots, g_1, h_1, h_0] \in \DL_{t-1}[\Gcal; \Hcal]$.
(We identify $\DL_0[\Gcal; \Hcal]$ with $\Hcal$.)

\begin{algorithm}[t]
\caption{\prepend}
\label{alg:prepend}
\begin{algorithmic}
\REQUIRE Groups $\Gcal$, hypothesis class $\Hcal$, i.i.d.\ examples $(x_1, y_1), \ldots, (x_n, y_n)$ from $\Dcal$, error bound $\epsilon_n \colon \Gcal \to \RR_+$.
\ENSURE Decision list $f_T \in \DL_T[\Gcal; \Hcal]$.
\STATE Compute $h_0 \in \argmin_{h \in \Hcal} L_n(h)$
\STATE Set $f_0 = [h_0] \in \DL_0[\Gcal; \Hcal]$.
\FOR{$t=0, 1, \ldots,$}
	\STATE Compute \[ \hspace{-1em} (g_{t+1}, h_{t+1}) \in \argmax_{\mathclap{(g,h) \in \Gcal \times \Hcal}} L_n(f_{t} \mid g) - L_n(h \mid g) - \epsilon_n(g). \]
	\IF{$L_n(f_{t} \mid g_{t+1}) - L_n(h_{t+1} \mid g_{t+1}) \geq \epsilon_n(g_{t+1})$}
		\STATE Prepend $(g_{t+1}, h_{t+1})$ to $f_{t}$ to obtain \[f_{t+1} := [g_{t+1}, h_{t+1}, g_{t}, h_{t}, \ldots, g_1, h_1, h_0].\]
	\ELSE
    	\STATE \textbf{return} $f_t$.
	\ENDIF
\ENDFOR
\end{algorithmic}
\end{algorithm}

We propose a simple algorithm, called \prepend (\pref{alg:prepend}),
for learning these decision lists.
\prepend proceeds in rounds, maintaining a current decision list $f_t \in \DL_t[\Gcal; \Hcal]$. At each round, it searches for a group $g_{t+1} \in \Gcal$ and a hypothesis $h_{t+1} \in \Hcal$ that witnesses an empirical violation of Eq.~\eqref{eqn:hs_obj_rel}. If such a violation is found, $f_t$ is updated to $f_{t+1}$ by prepending the pair $(h_{t+1}, g_{t+1})$ to the front of $f_t$. If no violation is found, then we claim that $f_t$ is good enough, and terminate.

We first bound the number of iterations executed by \pref{alg:prepend} before it terminates.
\begin{lemma}
\label{lem:prepend-convergence}
Suppose that every $g \in \Gcal$ satisfies $P_n(g) \cdot \epsilon_n(g)  \geq \epsilon_o$ and say that $\min_{h \in \Hcal} L_n(h) \leq \alpha$. Then \pref{alg:prepend} terminates after at most $t \leq \alpha/\epsilon_o$ rounds and outputs a predictor $f_t \in \DL_t[\Gcal; \Hcal]$ such that 
\[
  L_n(f_t \mid g) \leq \inf_{h \in \Hcal} L_n(h \mid g) + \epsilon_n(g)
  \quad \text{for all $g \in \Gcal$} .
\]
\end{lemma}

\subsection{Sample complexity}

The key step in bounding the sample complexity of \pref{alg:prepend} is in controlling the complexity of $\DL_T[\Gcal; \Hcal]$, whereupon \pref{thm:conv_conditional_losses} can be applied. To see how this is done, consider the case where $|\Gcal|$ and $|\Hcal|$ are finite. In this setting, there are $T$ decision nodes, each of which can be chosen from $\Gcal$, and there are $T+1$ prediction nodes chosen from $\Hcal$. Thus, $|\DL_T[\Gcal; \Hcal]| \leq |\Gcal|^T |\Hcal|^{T+1}$.

To apply this observation, we first note that for any $f = [g_T, h_T, \ldots, g_1, h_1, h_0] \in \DL_T[\Gcal; \Hcal]$,
if there are rounds $t < t'$ such that $g_t = g_{t'}$, then $f$ is functionally equivalent to $f' \in \DL_{T-1}[\Gcal; \Hcal]$ where $f'$ simply removes the occurrence of $h_t, g_t$ in $f$. Thus, when the number of groups is finite, we can always pretend as if $\DL_T[\Gcal; \Hcal]$ is the set of decision lists of length \emph{exactly} $|\Gcal|$. The next result follows immediately.

\begin{proposition}
\label{prop:small-strata-complexity}
Suppose $|\Gcal|$ and $|\Hcal|$ are finite. The following holds with probability at least $1-\delta$. If \pref{alg:prepend} is run until convergence, it will terminate with a predictor $f$ that satisfies
\begin{multline*}
  L(f \mid g) \ \leq \ \min_{h \in \Hcal} L(h \mid g) + \epsilon_n(g) \\
  + O\left( \sqrt{\frac{ |\Gcal| \log(|\Hcal| |\Gcal|) + \log(1/\delta)} {\#_n(g)}} \right)
\end{multline*}
for all $g \in \Gcal$.
\end{proposition}

\pref{prop:small-strata-complexity} suggests that when $|\Gcal|$ is small, a reasonable approach is to take 
\[ \epsilon_n(g) \ = \ O \left(\sqrt{ \frac{|\Gcal| \log (|\Hcal| |\Gcal|) + \log(1/\delta)}{\#_n(g)}}\right),\]
in which case we will terminate with a predictor whose excess conditional error on any group $g$ is $O(\epsilon_n(g))$. Thus, when the number of groups is small, estimation error and learning error are within a factor of $\sqrt{|\Gcal|}$.

If the number of groups is very large, or infinite, \pref{prop:small-strata-complexity} may be vacuous. However, if we have a lower bound on the empirical mass of any group (or perhaps restrict ourselves to such groups), then we can give a result that remains useful. To do so, we introduce the notation
\[ \Gcal_{n, \gamma} \ = \ \{ g \in \Gcal \, : \, \#_n(g) \geq \gamma  n\}. \]
Given this notation, we have the following result.

\begin{theorem}
\label{thm:large-strata-complexity}
Suppose that $\Hcal$ and $\Gcal$ are finite, and $\gamma>0$ is given. There is a setting of $\epsilon_n(\cdot)$ such that the following holds. If \pref{alg:prepend} is run with groups $\Gcal_{n, \gamma}$, then with probability $1-\delta$, it terminates with a predictor $f$ satisfying
\[
  L(f \mid g) \leq \min_{h \in \Hcal} L(h \mid g) + O \left( \sqrt[3]{ \frac{K/\gamma}{\#_n(g)} } + \sqrt{\frac{\log(1/\delta)}{\#_n(g)}} \right)
\]
for all $g \in \Gcal_{n, \gamma}$, where $K := \log(|\Gcal| ||\Hcal|)$. In addition, with probability $1-\delta$, all $g \in \Gcal$ satisfying $P(g) \geq \gamma + \sqrt{\frac{\log(|\Gcal|/\delta)}{n}}$ will lie in $\Gcal_{n, \gamma}$.
\end{theorem}

\pref{thm:large-strata-complexity} implies the following sample complexity to achieve the error bound of the type in Eq.~\eqref{eqn:hs_obj_abs}.

\begin{corollary}
\label{corr:large-strata-abs}
There is an absolute constant $c>0$ such that the following holds. Suppose that $\Hcal$ and $\Gcal$ are finite, and $\gamma := \min_{g \in \Gcal} P(g)$. There is a setting of $\epsilon_n(\cdot)$ such that the following holds. If 
\[ n \ \geq \ \frac{c}{\epsilon^3 \gamma^2} \log \left(\frac{|\Gcal| |\Hcal|}{\delta} \right) \]
then with probability $1-\delta$, \pref{alg:prepend} terminates with a predictor $f$ satisfying
\[
  L(f \mid g) \ \leq \ \min_{h \in \Hcal} L(h \mid g) + \epsilon
  \quad \text{for all $g \in \Gcal$} .
\]
\end{corollary}

\prepend can also handle the setting where both the number of groups and the number of hypotheses are infinite, so long as the {pseudo-dimension} of $\ell \circ \Hcal$ and the VC-dimension of $\Gcal$ are bounded.

\begin{theorem}
\label{thm:infinite-strata-complexity}
Suppose that both the pseudo-dimension of $\ell \circ \Hcal$ and the VC-dimension of $\Gcal$ are bounded above by $d$, and $\gamma>0$ is given. There is a setting of $\epsilon_n(\cdot)$ such that the following holds. If \pref{alg:prepend} is run with groups $\Gcal_{n, \gamma}$, then with probability $1-\delta$, it returns a predictor $f$ satisfying
\[
  L(f \mid g) \leq \min_{h \in \Hcal} L(h \mid g) + O \left( \sqrt[3]{\frac{ d\log n}{\gamma \#_n(g)}} + \sqrt{\frac{\log(1/\delta)}{\#_n(g)}} \right)
\]
for all $g \in \Gcal_{n, \gamma}$. Moreover, with probability $1-\delta$, all $g \in \Gcal$ satisfying $P(g) \geq \gamma + \sqrt{\frac{d\log(2n) + \log(4/\delta)}{n}}$ lie in $\Gcal_{n, \gamma}$.
\end{theorem}

\subsection{Inadmissibility of evaluation functions}

One interpretation of decision lists in the class $\DL_t[\Gcal; \Hcal]$ is that they correspond to orderings of the set $\Gcal \times \Hcal$. That is, for any $f \in \DL_t[\Gcal; \Hcal]$ there is a corresponding ordering $(g_1, h_1), (g_2, h_2), \ldots$ where for a given $x \in \Xcal$, we first find 
\[ i(x) \ = \ \min\{ i \, : \, g_i(x) = 1 \} \]
and classify according to $h_{i(x)}(x)$. Given this alternate perspective, one can ask whether it is possible to calculate such orderings directly. We will show a negative result here for the approach of {evaluation functions}.

An evaluation function $s: \Hcal \times \Gcal \times (X\cal \times \Ycal)^* \rightarrow \RR$ takes as input $h\in \Hcal$, $g \in \Gcal$ and a sample $(x_1, y_1), \ldots, (x_n, y_n)$ and outputs a real number as a score. The ordering induced by an evaluation function is then simply the ordering of the corresponding scores, with ties being broken by some (possibly randomized) rule. We say that $s$ is an order-1 evaluation function if $s(h, g, (x_1, y_1), \ldots, (x_n, y_n))$ is only a function of $n$, $P_n(g)$ and $L_n(h \mid g)$.

By \pref{thm:conv_conditional_losses}, $P_n(g)$ and $L_n(h \mid g)$ converge to their expectations as $n$ grows to infinity. Thus, in the limit, an order-1 evaluation function is a function of $P(g)$ and $L(h \mid g)$. Unfortunately, there exist two scenarios where these statistics are identical but no decision list solves the multi-group learning problem for both scenarios simultaneously.

\begin{proposition}
\label{prop:neg-scoring-functions}
There exist $\Hcal = \{ h_1, h_2 \}$, $\Gcal = \{g_1, g_2 \}$, and distributions $\Dcal_1$ and $\Dcal_2$ such that the following holds.
\begin{itemize}
	\item $P_{\Dcal_1}(g) = P_{\Dcal_2}(g)$ and $L_{\Dcal_1}(h \mid g) = L_{\Dcal_2}(h \mid g)$ for all $h \in \Hcal$, $g \in \Gcal$.
	\item For any decision list $f \in \DL_t[\Gcal; \Hcal]$, there exists an $i \in \{ 1, 2\}$ and $g \in \Gcal$ such that
	\[ L_{\Dcal_i}(f \mid g) \  \geq \ \min_{h \in \Hcal} L_{\Dcal_i}(h \mid g) + \frac{1}{8}.  \]
\end{itemize}
Here, the subscript $\Dcal_i$ denotes taking probabilities with respect to $\Dcal_i$.
\end{proposition}

\section{A reduction to sleeping experts}
\label{sec:sleeping-experts}
\begin{algorithm}[t]
\caption{Reduction to sleeping experts}
\label{alg:sleeping-experts}
\begin{algorithmic}
  \REQUIRE Groups $\Gcal$, hypothesis class $\Hcal$, $2n$ i.i.d.\ examples $(x_1, y_1), \ldots, (x_{n}, y_{n})$, $(x'_1, y'_1), \ldots, (x'_{n}, y'_{n})$ from $\Dcal$.
\ENSURE Randomized predictor $Q$.
\STATE Run \mlchedge on $(x_1, y_1), \ldots, (x_{n}, y_{n})$ using experts $\Hcal \times \Gcal$ with uniform initial probabilities, and per-expert learning rates $\eta_{h,g}= \min \bigl\{ \sqrt{ \frac{ \log(|\Hcal| |\Gcal|)}{\sum_{i=i}^{n} g(x'_i)}} , 1\bigr\}$. Let $p_1, \ldots, p_n$ be the internal hypotheses of \mlchedge.
\OUTPUT $Q = $ uniform distribution over $p_1, \ldots, p_n$.
\end{algorithmic}
\end{algorithm}

In this section, we will show that the rate suggested by \pref{thm:conv_conditional_losses} is achievable via a reduction to the sleeping experts problem, similar to the result of \citet{blum2020advancing}. The sleeping experts problem is an online expert aggregation problem such that during every round, some experts are `awake' and some are `asleep'~\citep{blum1997empirical,freund1997using}. The goal for a learner in this setting is to achieve low regret against every expert on those rounds in which it was awake. %

To reduce the multi-group learning problem to a sleeping experts problem, we create an expert for every pair $(h, g) \in \Hcal \times \Gcal$. In each round $t$, we feed an example $x_t$ and say that expert $(h,g)$ is awake if and only if $g(x_t)=1$, in which case the expert prediction is $h(x_t)$ and the revealed loss is $\ell(h(x_t), y_t)$. Formally, the reduction looks as follows:

\begin{itemize}
	\item[] For $t=1,2, \ldots, n$:
	\begin{itemize}
		\item Input $x_t$ is revealed.
		\item If $g(x_t)=1$, then expert $(h,g)$ is awake and predicts $h(x_t)$. Otherwise, $(h,g)$ is asleep.
		\item The learner produces a distribution $p_t$ over the experts such that $\sum_{h,g} g(x_t) p_t(h,g) = 1$.
		\item Label $y_t$ is revealed, and the learner suffers loss $\hat{\ell}_t = \sum_{h,g} p_t(h,g) g(x_t) \ell(h(x_t), y_t)$.
	\end{itemize}
\end{itemize}

There are several suitable algorithms in the literature for the sleeping
experts problem. Most convenient for our purposes is \mlchedge, originally
proposed by \citet{blum2007external} and further analyzed by
\citet{gaillard2014second}. In \pref{app:sleeping-experts}, we present 
\mlchedge and state the learning guarantees proven by \citet{gaillard2014second}.

To convert this online learner into a batch learning algorithm,
we follow the strategy of \citet{helmbold1995weak}.
We do this by keeping track of the internal hypotheses of the online learner.
That is, let $p_t( \cdot \, ; \, x)$ be the distribution that the learner at time $t$ would produce if fed the example $x$. Given these internal hypotheses, the final predictor is as follows: on input $x$, draw $t$ uniformly at random from $\{ 1, \ldots, n\}$, draw $(h,g)$ from $p_t(\cdot \, ; \, x)$, and predict $h(x)$. 

For a collection of internal hypotheses $p_1, \ldots, p_n$ and a distribution $Q$ over such hypotheses, we use the notation 
\begin{align*}
L(p_t \mid g) \ &:= \ \EE_{(x,y)} \left[ \EE_{\left(\tilde{h}, \tilde{g}\right) \sim p_t( \cdot \,  ; \, x)} \left[ \ell\left(\tilde{h}(x), y\right) \right] \bigm\vert g \right] \\
L(Q \mid g) \ &:= \ \EE_{p_t \sim Q} \left[ L(p_t \mid g) \right].
\end{align*}

To run \mlchedge, we need to specify an initial probability distribution over the experts and a set of expert-specific learning rates. Our initial distribution is uniform distribution over the experts. For the learning rates, we use half the sample to estimate the empirical probability masses of each group and set the learning rates to optimize an upper bound on the per-expert regret~\citep[Theorem~16]{gaillard2014second}. %
\pref{alg:sleeping-experts} presents the full method.

\begin{theorem}
\label{thm:sleeping-experts}
Let $Q$ be the randomized predictor returned by \pref{alg:sleeping-experts}. Then
\[
  L(Q \mid g) \ \leq \ \min_{h \in \Hcal} L(h \mid g)  + O\left( \sqrt{\frac{\log( |\Gcal||\Hcal|/\delta )}{\#_n(g)}} \right)
\]
for all $g \in \Gcal$ with probability at least $1-\delta$.
\end{theorem}

\pref{thm:sleeping-experts} implies the following corollary.

\begin{corollary}
\label{corr:sleeping-experts-abs}
There is an absolute constant $c > 0$ such that the following holds. Suppose $P(g) \geq \gamma > 0$ for all $g \in \Gcal$. Let $Q$ be the randomized predictor returned by \pref{alg:sleeping-experts}. If
\[ n \ \geq \frac{c}{\epsilon^2 \gamma}  \log \frac{|\Gcal||\Hcal|}{ \delta} , \] 
then with probability at least $1-\delta$,
\[
  L(Q \mid g) \ \leq \ \min_{h \in \Hcal} L(h \mid g)  + \epsilon
  \quad \text{for all $g \in \Gcal$} .
\]
\end{corollary}

The dependence on $\log|\Hcal|$ in \pref{thm:sleeping-experts} and
\pref{corr:sleeping-experts-abs} can be replaced by the pseudo-dimension of
$\ell \circ \Hcal$ (up to a $\log n$ factor) using a slight
modification to \pref{alg:sleeping-experts}.
Simply use half of the training data to find, for each $g \in \Gcal$, a hypothesis $\hat h_g \in \Hcal$ satisfying
\[
  L(\hat h_g \mid g) \leq \inf_{h \in \Hcal} L(h \mid g)
  + O\left( \sqrt{\frac{d \log n + \log(|\Gcal|/\delta)}{\#_n(g)}} \right)
\]
(say, using ERM); and then execute
\pref{alg:sleeping-experts} using $\{ \hat h_g : g \in \Gcal \}$ instead of
$\Hcal$, along with the other half of the training data.
Removing the dependence on $\log |\Gcal|$ is algorithmically less straightforward.

\section{The group-realizable setting}
\label{sec:realizable}

We restrict our attention now to the case where $\Ycal = \Zcal = \{ -1, +1\}$ and $\ell(z,y) = \ind[z\neq y]$ is the binary zero-one loss. In the \emph{group-realizable} setting, each group has an associated perfect classifier:
\[ \min_{h \in \Hcal} L(h \mid g) \ = \ 0 \quad \text{for all $g \in \Gcal$} .\]
Note that \emph{realizability} is the stronger assumption $\min_{h \in \Hcal} L(h) = 0$.
Realizability implies group-realizability, but not vice versa.
In this setting, the arguments from \pref{sec:sleeping-experts} can be adapted to show that the randomized predictor produced by \pref{alg:sleeping-experts} achieves error
\[  L(Q \mid g) \ \leq \ O\left( \frac{\log(|\Gcal||\Hcal|)}{\#_n(g)} \right) \]
with high probability for all $g \in \Gcal$. %
One simply uses $\eta_{h,g} = 1/2$ for all $h \in \Hcal$ and $g \in \Gcal$.

(In an earlier version of the paper, we claimed that a simple majority-based algorithm, using hypotheses of the form $x \mapsto \sign\left( \sum_{g \in \Gcal} g(x) \hat{h}_g(x) \right)$, achieves the same guarantee as \pref{alg:sleeping-experts}, but the analysis was incorrect.)

\section{Discussion and open problems}

In this work, we presented simple and near-optimal algorithms for multi-group
learning.
Here we point to some interesting directions for future work.
\begin{description}
	\item{\textbf{Computation.}} It is not clear if any of the algorithms considered in this work can be made efficient, as they seem to rely either on enumeration or on complicated optimization subroutines. Thus, it is an interesting open problem to devise multi-group learning algorithms that are efficient for some specific choices of $\Hcal$ and $\Gcal$ and some restrictions on the marginal distribution over $\Xcal$. %
	\item{\textbf{Representation.}} Through \pref{alg:prepend}, we have addressed the question of
representational complexity for multi-group learning: we showed that decision
lists of the form in $\DL_T[\Gcal;\Hcal]$ are sufficient.
But in some cases the length of the decision list can be much smaller than $|\Gcal|$.
An expanded investigation of these representational issues may give a sharper quantative bound.
For example, what is needed in the group-realizable setting?
	\item{\textbf{Simplicity and optimality.}} Finally, it remains an interesting open problem is to design an algorithm that is simple like \pref{alg:prepend} but that also enjoys the performance guarantees of \pref{alg:sleeping-experts}.
\end{description}

%% file: appendix.tex
\section{Missing proofs from \pref{sec:setting}}
\label{app:setting}

\subsection{Proof of \pref{thm:conv_conditional_losses}}

To simplify the proof, we will use the following notation. Let $P$ be a probability distribution over $\Xcal$, let $\Hcal$ be a family of $[0,1]$-valued functions over $\Xcal$, and let $\Gcal$ be a family of $\{0,1\}$-valued functions over $\Xcal$. Given a sample $x_1, \ldots, x_n$ drawn from $P$, we make the following definitions:
\begin{align*}
P(h \mid g) \ &:= \ \frac{P(hg)}{P(g)} \ := \  \frac{\EE[h(x) g(x)]}{\EE[g(x)]} \\
P_n(h \mid g) \ &:= \ \frac{P_n(hg)}{P_n(g)} \ := \  \frac{\sum_{i=1}^n h(x_i) g(x_i)}{\#_n(g)}.
\end{align*}
Given this notation, we will prove the following theorem, which directly implies \pref{thm:conv_conditional_losses}.
\begin{theorem}
\label{thm:condition-expectation-shattering}
Let $P$ be a probability distribution over $\Xcal$, let $\Hcal$ be a family of $[0,1]$-valued functions over $\Xcal$, and let $\Gcal$ be a family of $\{0,1\}$-valued functions over $\Xcal$. Then with probability at least $1-\delta$,
\[ \left| P(h \mid g) - P_n(h \mid g) \right| \ \leq \ \min\left\{ 9 \sqrt{\frac{D}{\#_n(g)}}, 7 \sqrt{ \frac{ D P_n(h \mid g) }{\#_n(g)}  } + \frac{16 D}{\#_n(g)}  \right\} \]
for all $h \in \Hcal, g \in \Gcal$, where $D = 2\log \Pi_{2n}( \Hcal_{\thresh}) + 2\log \Pi_{2n}(\Gcal) + \log(8/\delta)$.
\end{theorem}

To prove \pref{thm:condition-expectation-shattering}, we will provide relative deviation bounds on $[0,1]$-valued functions of the following form: with probability $1-\delta$, 
\[ |P(h) - P_n(h)| \ \leq \ \sqrt{\frac{P_n(h) \text{comp}(\Hcal) \log(1/\delta)}{n}} + \frac{\text{comp}(\Hcal) \log(1/\delta)}{n} \]
for all $h \in \Hcal$, where $\text{comp}(\Hcal)$ is some complexity measure of $\Hcal$. 

To establish this relative deviation bound, we first reduce the problem from $[0,1]$-valued functions to $\{0,1\}$-valued functions.
\begin{lemma}
\label{lem:reduction-to-thresh}
If $\Hcal$ is a class of $[0,1]$-valued function, then
\begin{align*}
\pr\left( \sup_{h \in \Hcal} \frac{P(h) - P_n(h)}{\sqrt{P(h)}} > \epsilon \right) & \leq \ \pr\left( \sup_{h' \in \Hcal_\thresh} \frac{P(h') - P_n(h')}{\sqrt{P(h')}} > \epsilon \right) \\
\pr\left( \sup_{h \in \Hcal} \frac{P_n(h) - P(h)}{\sqrt{P_n(h)}} > \epsilon \right) & \leq \ \pr\left( \sup_{h' \in \Hcal_\thresh} \frac{P_n(h') - P(h')}{\sqrt{P_n(h')}} > \epsilon \right).
\end{align*}
\end{lemma}
\begin{proof}
The proof is inspired by the proof of Theorem~3 of \citet{cortes2019relative}. For $h \in \Hcal, t \in [0,1]$, we use the notation
\begin{align*}
P(h > t) \ &= \ \EE[ \ind[h(x) > t]] \\
P_n(h > t) \ &= \ \frac{1}{n} \sum_{i=1}^n \ind[h(x_i) > t].
\end{align*}
To prove the lemma, it will suffice to prove the following two implications
\begin{align*}
\forall {h \in \Hcal, t \in [0,1]}, \, P(h > t) - P_n(h > t) \leq \epsilon \sqrt{P(h > t)} \ & \longrightarrow  \ \forall h \in \Hcal, \,  P(h) - P_n(h) \leq \epsilon \sqrt{P(h)} \\
\forall {h \in \Hcal, t \in [0,1]}, \, P_n(h > t) - P(h > t)  \leq \epsilon \sqrt{P_n(h > t)} \ & \longrightarrow  \ \forall h \in \Hcal, \, P_n(h) - P(h) \leq \epsilon \sqrt{P_n(h)}.
\end{align*}
We will prove the first statement, as the second can be proven symmetrically. Assume that 
\[P(h > t) - P_n(h > t) \leq \epsilon \sqrt{P(h > t)}\] 
for all $h \in \Hcal$ and $t \in [0,1]$. Since $h$ is $[0,1]$-valued, we can write
\begin{align*}
P(h) \ &= \ \int_{0}^1 P(h > t) \, \dif t \\
P_n(h) \ &= \ \int_{0}^1 P_n(h > t) \, \dif t.
\end{align*}
Thus,
\begin{align*}
P(h) - P_n(h) \ &= \ \int_{0}^1 P(h > t) - P_n(h > t) \, \dif t \\
\ &\leq \ \int_{0}^1\epsilon \sqrt{P(h > t)} \, \dif t \\
\ &\leq \ \epsilon \sqrt{\int_0^1 P(h > t) \, \dif t } \\
\ &= \ \epsilon \sqrt{P(h)}
\end{align*}
where the first inequality is by assumption and the second is by Jensen's inequality.
\end{proof}

The following result, found for example in~\citet{boucheron2005theory}, provides uniform convergence bounds on the relative deviations of $\{ 0,1\}$-valued functions.
\begin{lemma}
\label{lem:BBL}
If $\Fcal$ is a class of $\{0,1\}$-valued function, then
\begin{align*}
\pr \left( \sup_{f \in \Fcal} \frac{P_n (f) - P(f)}{\sqrt{P_n(f)}} > \epsilon \right)  \ &\leq \  4 \Pi_{2n}(\Fcal)\exp \left(- \frac{\epsilon^2 n}{4} \right) \\
\pr \left( \sup_{f \in \Fcal} \frac{P (f) - P_n(f)}{\sqrt{P(f)}} > \epsilon \right)  \ &\leq \  4 \Pi_{2n}(\Fcal) \exp \left(- \frac{\epsilon^2 n}{4} \right)
\end{align*}
where $\Pi_k(\Fcal)$ is the $k$-th shattering coefficient of $\Fcal$.
\end{lemma}

Combining these two results, we have the following.
\begin{lemma}
\label{lem:relative-deviations-0,1}
If $\Hcal$ is a class of $[0,1]$-valued function, then with probability $1-\delta$,
\begin{align*}
P(h) - P_n(h) \ &\leq \ 2 \sqrt{P_n(h) \frac{\log \Pi_{2n}(\Hcal_\thresh)  + \log(8/\delta)}{n}} + 4  \frac{\log \Pi_{2n}(\Hcal_\thresh) + \log(8/\delta)}{n} \\
P(h) - P_n(h) \ & \geq \ - 2 \sqrt{P_n(h) \frac{\log \Pi_{2n}(\Hcal_\thresh)  + \log(8/\delta)}{n}}
\end{align*}
for all $h \in \Hcal$.
\end{lemma}
\begin{proof}
Combining Lemmas~\ref{lem:reduction-to-thresh} and~\ref{lem:BBL}, we immediately have that with probability $1-\delta$
\begin{align*}
\frac{P_n (h) - P(h)}{\sqrt{P_n(h)}} \ & \leq \ 2 \sqrt{\frac{\log \Pi_{2n}(\Hcal_\thresh) + \log(8/\delta)}{n}} \\
\frac{P(h) - P_n(h)}{\sqrt{P(h)}} \ & \leq \ 2 \sqrt{\frac{\log \Pi_{2n}(\Hcal_\thresh) + \log(8/\delta)}{n}}
\end{align*}
for all $h \in \Hcal$. Let us condition on this occurring. 

Using the standard trick that for $A,B,C \geq 0$, the inequality $A \leq B \sqrt{A} + C$ entails the inequality $A \leq B^2 + B \sqrt{C} + C$, we can observe from the second inequality above that
\[ P(h) \ \leq \ P_n(h) + 2 \sqrt{P_n(h) \frac{\log \Pi_{2n}(\Hcal_\thresh) + \log(8/\delta)}{n}} + 4  \frac{\log \Pi_{2n}(\Hcal_\thresh) + \log(8/\delta)}{n}.  \]
Combined with the first inequality, we have the lemma statement.
\end{proof}

Now we turn to the proof of Theorem~\ref{thm:condition-expectation-shattering}.
\begin{proof}[Proof of Theorem~\ref{thm:condition-expectation-shattering}]
The proof is similar to the proof of Theorem~5 of \citet{BDFM19}. Let $\Fcal = \Gcal \cup \{ hg \, : \, h \in \Hcal, g \in \Gcal \}$. Note that for $g \in \Gcal$, $h \in \Hcal$ and $t \in \RR$, we have
\[ \ind[h(x) g(x) > t] = g(x) \ind[h(x) > t]. \]
Thus, if we let $\Ccal = \{ hg \, : \, h \in \Hcal, g \in \Gcal \}$, we observe that
\begin{align*}
\log \Pi_n(\Fcal_\thresh) \ &\leq \  \log \Pi_n(\Gcal) + \log \Pi_n(\Ccal_\thresh)  \\
\ &\leq \  \log\Pi_n(\Gcal) + \log \Pi_n(\Hcal_\thresh) \Pi_n(\Gcal) \\
\ &\leq \ 2 \log \Pi_n(\Hcal_\thresh) \Pi_n(\Gcal).
\end{align*}
Combining this with Lemma~\ref{lem:relative-deviations-0,1}, we have with probability $1-\delta$
\begin{align*}
 P_n(f) - 2 \sqrt{P_n(f) \frac{D}{n}}  \ \leq \ P(f)  \ \leq \ P_n(f) + 2 \sqrt{P_n(f) \frac{D}{n}} + 4  \frac{D}{n}
\end{align*}
for all $f \in \Fcal$, where we used the definition $D = 2\log\left(\Pi_{2n}(\Hcal_\thresh)\right) + 2\log\left(\Pi_{2n}(\Gcal)\right) + \log(8/\delta)$. Let us condition on this event holding.

Now fix some $h \in \Hcal$ and $g \in \Gcal$. We can work out
\begin{align*}
P(h \mid g) \ &= \ \frac{P(h g)}{P(g)} \\
\ &\leq \ \frac{P_n(h g) + 2 \sqrt{P_n(hg) \frac{D}{n}} + 4 \frac{D}{n}}{P_n(g) - 2 \sqrt{P_n(g) \frac{D}{n}}} \\
\ &= \ \frac{P_n(hg)}{P_n(g)} \cdot \frac{1 + 2 \sqrt{\frac{D}{n P_n(hg)}} + 4 \frac{D}{n P_n(hg)}}{1 - 2 \sqrt{\frac{D}{n P_n(g)}}} \\
\ &= \ \frac{\#_n(hg)}{\#_n(g)} \cdot \frac{1 + 2 \sqrt{\frac{D}{\#_n(hg)}} + 4 \frac{D}{\#_n(hg)}}{1 - 2 \sqrt{\frac{D}{\#_n(g)}}}.
\end{align*}
Here, we have used the notation $\#_n(f) = n P_n(f)$. Observe that if $\#_n(g) \leq 16 D$, the theorem statement is trivial. Thus, we may assume that $\#_n(g) > 16 D $, whereupon the inequality $\frac{1}{1-x} \leq 1 + 2x$ for all $x < 1/2$ implies
\[ \frac{1}{1 - 2 \sqrt{\frac{D}{\#_n(g)}}} \ \leq \ 1 +4\sqrt{\frac{D}{\#_n(g)}} .   \]
Thus, we have
\begin{align*}
P(h \mid g) \ & \leq \ \frac{\#_n(hg)}{\#_n(g)} \cdot \left( 1 + 2 \sqrt{\frac{D}{\#_n(hg)}} + 4 \frac{D}{\#_n(hg)} \right) \left( 1 +4\sqrt{\frac{D}{\#_n(g)}}\right) \\
\ &= \ \frac{\#_n(hg)}{\#_n(g)} \cdot \left( 1 + 2 \sqrt{\frac{D}{\#_n(hg)}} + 4 \frac{D}{\#_n(hg)} + 4\sqrt{\frac{D}{\#_n(g)}} + 4 \frac{D}{\sqrt{\#_n(hg) \#_n(g)}} + 16 \frac{D^{3/2}}{\#_n(hg) \sqrt{\#_n(g)}} \right)  \\
\ &= \ \frac{\#_n(hg)}{\#_n(g)} + 2\frac{\sqrt{D\#_n(hg)}}{\#_n(g)}  + 4 \frac{D}{\#_n(g)} + 4 \frac{\#_n(hg)}{\#_n(g)}\sqrt{\frac{D}{\#_n(g)}} + 4\frac{D \sqrt{\#_n(hg)}}{\#_n(g)^{3/2}} + 16 \frac{D^{3/2}}{ \#_n(g)^{3/2}} \\
\ &\leq \ P_n(h \mid g) + \min\left\{ 9 \sqrt{\frac{D}{\#_n(g)}}, 7 \sqrt{ \frac{ D P_n(h \mid g) }{\#_n(g)}  } + \frac{8 D}{\#_n(g)}  \right\}
\end{align*}
where we have made use of the inequalities $\#_n(hg) \leq \#_n(g)$ and $\#_n(g) > 16 D$.

In the other direction, we have two cases: $\#_n(hg) < 4D$ and $\#_n(hg) \geq 4D$. Let us assume first that $\#_n(hg) < 4D$. Then observe that 
\begin{align*}
P_n(h \mid g) - 9 \sqrt{\frac{D}{\#_n(g)}} 
\ & = \ \frac{\#_n(hg)}{\#_n(g)} - 9 \sqrt{\frac{D}{\#_n(g)}}  \\
\ &< \ \frac{4D}{\#_n(g)} - 9 \sqrt{\frac{D}{\#_n(g)}}  \\
\ &= \ \sqrt{\frac{D}{\#_n(g)}} \left( 4\sqrt{\frac{D}{\#_n(g)}} - 9\right)
\ \leq \ 0 \ \leq \ P(h \mid g)
\end{align*}
where we have used the fact that $\#_n(g) > 16 D$. Similarly, we also have
\begin{align*}
P_n(h \mid g) - 7 \sqrt{\frac{D P_n(h \mid g)}{\#_n(g)}} - \frac{8D}{\#_n(g)} 
\ & = \ \frac{\#_n(hg)}{\#_n(g)} -  7 \sqrt{\frac{D \#_n(hg)}{\#_n(g)^2}} - \frac{8D}{\#_n(g)}   \\
\ &< \ \frac{4D}{\#_n(g)} -  7 \sqrt{\frac{D \#_n(hg)}{\#_n(g)^2}} - \frac{8D}{\#_n(g)} 
\ \leq \ 0 \ \leq \ P(h \mid g)
\end{align*}

Thus, we may assume that $\#_n(hg) \geq 4D$, so that we have
\begin{align*}
P(h \mid g) \ &= \ \frac{P(h g)}{P(g)} \\
\ & \geq \ \frac{P_n(h g) - 2 \sqrt{P_n(hg) \frac{D}{n}}}{P_n(g) + 2 \sqrt{P_n(g) \frac{D}{n}} + 4 \frac{D}{n}} \\
\ &= \ \frac{\#_n(hg)}{\#_n(g)} \cdot \frac{1 - 2 \sqrt{\frac{D}{\#_n(hg)}}}{1 + 2 \sqrt{\frac{D}{\#_n(g)}} +4 \frac{D}{\#_n(g)}}.
\end{align*}
Using the inequality $\frac{1}{1+x} \geq 1 - x$ for all $x \geq 0$, we have
\begin{align*}
P(h \mid g) \ & \geq \ \frac{\#_n(hg)}{\#_n(g)} \left(1 - 2 \sqrt{\frac{D}{\#_n(hg)}} \right) \left( 1 - 2 \sqrt{\frac{D}{\#_n(g)}} - 4 \frac{D}{\#_n(g)} \right) \\
\ &\geq \ \frac{\#_n(hg)}{\#_n(g)} \left( 1 - 2 \sqrt{\frac{D}{\#_n(hg)}} - 2 \sqrt{\frac{D}{\#_n(g)}} - 4 \frac{D}{\#_n(g)}  \right) \\
\ &= \ \frac{\#_n(hg)}{\#_n(g)}  - 2\frac{\sqrt{D\#_n(hg)}}{\#_n(g)} - 2 \frac{\#_n(hg) \sqrt{D}}{\#_n(g)^{3/2}} - 4 \frac{D \#_n(hg)}{\#_n(g)^2} \\
\ &\geq \ P_n(h \mid g)  -  \min\left\{ 5 \sqrt{\frac{D}{\#_n(g)}}, 4 \sqrt{ \frac{ D P_n(h \mid g) }{\#_n(g)}  } + \frac{4 D}{\#_n(g)}  \right\}
\end{align*}
where we again made use of the inequalities $\#_n(hg) \leq \#_n(g)$ and $\#_n(g) > 16 D$.
\end{proof}

\section{Missing proofs from \pref{sec:simple-algorithm}}
\label{app:simple-algorithm}

\subsection{Proof of \pref{lem:prepend-convergence}}
We first note that if the algorithm terminates at round $t$, then we trivially must have
\[
  L_n(f_t \mid g) \leq \inf_{h \in \Hcal} L_n(h \mid g) + \epsilon_n(g)
  \quad \text{for all $g \in \Gcal$} .
\]
If the algorithm does not terminate at round $t$, then we must have found a pair $h_{t+1}, g_{t+1}$ such that
\[ L(f_{t} \mid g_{t+1}) - L(h_{t+1} \mid g_{t+1}) \geq \epsilon_n(g_{t+1}) , \]
and we must have prepended the pair $(h_{t+1}, g_{t+1})$ onto $f_t$ to create $f_{t+1}$. Observe that this prepending action implies that $f_{t+1}$ will agree with $h_{t+1}$ on $g_{t+1}$ and agree with $f_t$ everywhere else.
Thus, we have
\begin{align*}
L_n(f_t) - L_n(f_{t+1}) \ &= \ \EE\left[ \ell(f_{t}(x), y) -  \ell(f_{t+1}(x), y)  \right] \\
\ &= \ P_n(g_{t+1}) \EE\left[ \ell(f_{t}(x), y) -  \ell(h_{t+1}(x), y)  \mid g_{t+1}(x) = 1 \right] \\
\ &\geq \ P_n(g_{t+1}) \epsilon_n(g_{t+1}) \ \geq \ \epsilon_o.
\end{align*}
Thus, $L_n(f_t)$ decreases by $\epsilon_o$ at every update, and we have $L_n(f_0) \leq \alpha$. The theorem follows by combining these two observations.

\subsection{Proof of \pref{prop:small-strata-complexity}}

The discussion before the statement of \pref{prop:small-strata-complexity} implies that $|\DL_T[\Gcal; \Hcal]| \leq |\Hcal|^{|\Gcal| + 1} |\Gcal|^{|\Gcal|}$. Applying \pref{thm:conv_conditional_losses}, and utilizing the fact that for any finite class $\Hcal$, we have $\Pi_n(\Hcal) \leq |\Hcal|$, we have with probability at least $1-\delta$
\begin{align*}
|L(f \mid g) - L_n(f \mid g )|
\ &\leq \ 9 \sqrt{\frac{ 2\log|\DL_T[\Gcal; \Hcal]| + 2\log |\Gcal|  + \log(8/\delta)}{\#_n(g)}} \\
\ & \leq \ 9 \sqrt{\frac{ 2(|\Gcal| + 1) \log(|\Hcal| |\Gcal|) + \log(8/\delta)}{\#_n(g)}}
\end{align*}
for all $f \in \DL_T[\Gcal; \Hcal]$ and $g \in \Gcal$. Combined with \pref{lem:prepend-convergence}, we have the proposition statement.

\subsection{Proof of \pref{thm:large-strata-complexity}}
We will actually show the slightly stronger bound of
\[ L(f \mid g) \ \leq \ \min_{h \in \Hcal} L(h \mid g) + 22 \left( \frac{ \alpha \log(|\Gcal| |\Hcal|)}{\gamma \#_n(g)} \right)^{1/3}
+ 18\sqrt{\frac{\log(8/\delta)}{\#_n(g)}}
\]
where $\alpha = \min_{h \in \Hcal} L_n(h)$. The theorem follows from the fact that $\alpha \leq 1$.

We will take $\epsilon_n(g)$ to be the function
\[ \epsilon_n(g) \ = \ 36^{2/3} \left( \alpha \log(|\Gcal| |\Hcal|) \right)^{1/3} \left( \frac{n}{\gamma} \right)^{1/6} \left( \frac{1}{\#_n(g)} \right)^{1/2}. \]
Then the number of rounds \prepend takes is bounded as
\begin{align*}
T \ &\leq \ \frac{\alpha }{\min_{g \in \Gcal_{n, \gamma}} P_n(g) \epsilon_n(g)}  \\
\ &= \ \frac{\alpha }{36^{2/3} \left( \alpha  \log(|\Gcal| |\Hcal|) \right)^{1/3} ( {n}/{\gamma} )^{1/6} \min_{g \in \Gcal_{n, \gamma}} P_n(g) (\#_n(g))^{-1/2}} \\
\ &= \alpha ^{2/3} 36^{-2/3} \left(\log(|\Gcal| |\Hcal|) \right)^{-1/3} \left( \frac{n}{\gamma}\right)^{-1/6}  \cdot \left( \frac{n}{\gamma} \right)^{1/2} \\
\ &= \ \alpha ^{2/3} 36^{-2/3} \left(c \log(|\Gcal| |\Hcal|) \right)^{-1/3} \left( \frac{n}{\gamma}\right)^{1/3} .
\end{align*}
Observe that since $|\Gcal|$ and $|\Hcal|$ are finite, we have $|\DL_T[\Gcal; \Hcal]| \leq |\Hcal|^{T+1} |\Gcal|^T$. Combining \pref{thm:conv_conditional_losses} and \pref{lem:prepend-convergence}, we have with probability $1-\delta$
\begin{align*}
L(f \mid g) & \leq \ L_n(f \mid g) +  9 \sqrt{\frac{ 2(T+1) \log(|\Hcal| |\Gcal|) + \log(8/\delta)}{\#_n(g)}} \\
\ &\leq \ \min_{h \in \Hcal} L_n(h \mid g) + \epsilon_n(g) + 9 \sqrt{\frac{ 2(T+1) \log(|\Hcal| |\Gcal|) + \log(8/\delta)}{\#_n(g)}} \\
\ &\leq \ \min_{h \in \Hcal} L(h \mid g) + \epsilon_n(g) + 18 \sqrt{\frac{ 2(T+1) \log(|\Hcal| |\Gcal|) + \log(8/\delta)}{\#_n(g)}} \\
\ &\leq \ \min_{h \in \Hcal} L(h \mid g) + \epsilon_n(g) + 36 \sqrt{\frac{ T \log(|\Hcal| |\Gcal|)}{\#_n(g)}}
+ 18\sqrt{\frac{\log(8/\delta)}{\#_n(g)}}.
\end{align*}
Now by our bound on $T$, we have
\begin{align*}
36 \sqrt{\frac{ T \log(|\Hcal| |\Gcal|)}{\#_n(g)}}
\ &\leq \ 36\left(\frac{c \log (|\Hcal| |\Gcal|)}{\#_n(g)} \cdot \alpha ^{2/3} 36^{-2/3} \left( \log(|\Gcal| |\Hcal|) \right)^{-1/3} \left( \frac{n}{\gamma}\right)^{1/3} \right)^{1/2} \\
\ &= \ 36^{2/3} \left( \alpha  \log(|\Gcal| |\Hcal|) \right)^{1/3} \left( \frac{n}{\gamma} \right)^{1/6} \left( \frac{1}{\#_n(g)} \right)^{1/2} \ = \ \epsilon_n(g).
\end{align*}
Thus, we have 
\begin{align*}
L(f \mid g) \ &\leq \ \min_{h \in \Hcal} L_n(h \mid g) + 2\epsilon_n(g)
+ 18\sqrt{\frac{\log(8/\delta)}{\#_n(g)}}.
\end{align*}
Finally, observe that $n/\gamma \leq \#_n(g)/\gamma^2$ for all $g \in \Gcal_{n, \gamma}$. Thus, we have
\begin{align*}
\epsilon_n(g) \ &= \ 36^{2/3} \left( \alpha  \log(|\Gcal| |\Hcal|) \right)^{1/3} \left( \frac{n}{\gamma} \right)^{1/6} \left( \frac{1}{\#_n(g)} \right)^{1/2} \\
\ &\leq \  36^{2/3} \left(\alpha  \log(|\Gcal| |\Hcal|) \right)^{1/3} \left( \frac{\#_n(g)}{\gamma^2} \right)^{1/6} \left( \frac{1}{\#_n(g)} \right)^{1/2} \\
\ &\leq \  36^{2/3}  \left( \frac{ \alpha  \log(|\Gcal| |\Hcal|)}{\gamma \#_n(g)} \right)^{1/3}. \qed
\end{align*}

\subsection{Proof of \pref{corr:large-strata-abs}}

By \pref{lem:BBL}, we have with probability $1-\delta/2$ that
\[ P_n(g) \ \geq \ P(g) - 2 \sqrt{\frac{P(g)}{n} \log \left(16|\Gcal|/\delta \right)} \]
for all $g \in \Gcal$.
For the choice of $n$ in the statement, this implies $P_n(g) \geq \gamma/2$. 
By \pref{thm:large-strata-complexity}, we have with probability $1- \delta/2$,
\[ L(f \mid g) \ \leq \ \min_{h \in \Hcal} L(h \mid g) + 22 \left( \frac{  2\log(16|\Gcal| |\Hcal|/\delta)}{\gamma \#_n(g)} \right)^{1/3}\]
for all $g \in \Gcal_{n,\gamma/2}$. By a union bound, both of these hold with probability at least $1-\delta$, in which case we have $\Gcal_{n,\gamma/2} = \Gcal$. Plugging in our bound on $\#_n(g) = n P_n(g)$, we have
\begin{align*}
 L(f \mid g) \ &\leq \ \min_{h \in \Hcal} L(h \mid g) + 22 \left( \frac{  2\log(16|\Gcal| |\Hcal|/\delta)}{\gamma \#_n(g)} \right)^{1/3} \\
 \ &\leq \ \min_{h \in \Hcal} L(h \mid g) + 22 \left( \frac{  4\log(16|\Gcal| |\Hcal|/\delta)}{\gamma^2 n} \right)^{1/3} \\
 \ &\leq \ \min_{h \in \Hcal} L(h \mid g)  + \epsilon
\end{align*}
where the last line follows from the choice of $n$ in the corollary statement. \qed

\subsection{Proof of \pref{thm:infinite-strata-complexity}}
Observe first that every prediction by $f \in \DL_T[\Gcal; \Hcal]$ is actually done by some $h \in \Hcal$. Thus,
\begin{align*}
 \Pi_n( (\ell \circ \DL_T[\Gcal; \Hcal])_\thresh) \ &\leq \ \Pi_n( (\ell \circ \Hcal)_\thresh)^{T+1} \Pi_n(\Gcal)^T \\
 \ &\leq \ {n \choose \leq d}^{T+1} {n \choose \leq d}^T \\
 \ &\leq \ (2n)^{d(2T+1)}
\end{align*}
where the second line follows from the definition of pseudo-dimension and the Sauer-Shelah-Perles Lemma, and third line follows from the inequality ${a \choose \leq b} \leq (2a)^b$. Thus, with probability $1-\delta$ we have
\begin{equation}
\label{eqn:pseudo-bounds}
|L(f \mid g) - L_n(f \mid g )| \ \leq \ 18 \sqrt{\frac{ d(T+1)\log(2n) + \log(8/\delta)}{\#_n(g)}}
\end{equation}
for all $f \in  \DL_T[\Gcal; \Hcal]$ and $g \in \Gcal$. Let us condition on this occurring.

We will take $\epsilon_n(g)$ to be the function
\[ \epsilon_n(g) \ = \ \left(2 \cdot 36^2 d \log(2n) \right)^{1/3} \left( \frac{n}{\gamma} \right)^{1/6} \left( \frac{1}{\#_n(g)} \right)^{1/2}. \]
Then the number of rounds the rewrite algorithm takes is bounded as
\begin{align*}
T \ &\leq \ \frac{L_0}{\min_{g \in \Gcal_{n, \gamma}} P_n(g) \epsilon_n(g)}  \\
\ &= \ \frac{1}{\left(2 \cdot 36^2 d \log(2n) \right)^{1/3} ( {n}/{\gamma} )^{1/6} \min_{g \in \Gcal_{n, \gamma}} P_n(g) (\#_n(g))^{-1/2}} \\
\ &\leq \left(2 \cdot 36^2 d \log(2n) \right)^{-1/3} \left( \frac{n}{\gamma}\right)^{-1/6}  \cdot \left( \frac{n}{\gamma} \right)^{1/2} \\
\ &= \ \left(2 \cdot 36^2 d \log(2n) \right)^{-1/3} \left( \frac{n}{\gamma}\right)^{1/3} .
\end{align*}
Here, we have used the fact that the loss is bounded above by 1.
Eq.~\eqref{eqn:pseudo-bounds} implies that, for the function $f$ returned by \pref{alg:prepend}, we have for all $g \in \Gcal_{n,\gamma}$,
\begin{align*}
L(f \mid g) & \leq \ L_n(f \mid g) + 18 \sqrt{\frac{ d(T+1)\log(2n) + \log(8/\delta)}{\#_n(g)}}\\
\ &\leq \ \min_{h \in \Hcal} L_n(h \mid g) + \epsilon_n(g) + 18 \sqrt{\frac{ d(T+1)\log(2n) + \log(8/\delta)}{\#_n(g)}} \\
\ &\leq \ \min_{h \in \Hcal} L(h \mid g) + \epsilon_n(g) + 36 \sqrt{\frac{ d(T+1)\log(2n) + \log(8/\delta)}{\#_n(g)}} \\
\ &\leq \ \min_{h \in \Hcal} L(h \mid g) + \epsilon_n(g)
+ 36 \sqrt{\frac{ 2 T d \log(2n) }{\#_n(g)}}
+ 36 \sqrt{\frac{ \log(8/\delta)}{\#_n(g)}}.
\end{align*}
Now by our bound on $T$, we have
\begin{align*}
36 \sqrt{\frac{ 2 T d\log(2n)}{\#_n(g)}} 
\ &\leq \ 36\sqrt{ \frac{ 2 d\log(2n)}{\#_n(g)}  \left(2 \cdot 36^2 d\log(2n) \right)^{-1/3} \left( \frac{n}{\gamma}\right)^{1/3} }  \\
\ &= \ \left( 2 \cdot 36^2 d \log(2n) \right)^{1/3} \left( \frac{n}{\gamma} \right)^{1/6} \left( \frac{1}{\#_n(g)} \right)^{1/2} \\
\ &= \ \epsilon_n(g) .
\end{align*}
Thus, we have 
\begin{align*}
L(f \mid g) \ &\leq \ \min_{h \in \Hcal} L_n(h \mid g) + 2\epsilon_n(g)
+ 36 \sqrt{\frac{ \log(8/\delta)}{\#_n(g)}} .
\end{align*}
Finally, observe that $n/\gamma \leq \#_n(g)/\gamma^2$ for all $g \in \Gcal_{n, \gamma}$. Thus, we have
\begin{align*}
\epsilon_n(g) \ &= \ \left(2 \cdot 36^2 d \log(2n) \right)^{1/3} \left( \frac{n}{\gamma} \right)^{1/6} \left( \frac{1}{\#_n(g)} \right)^{1/2} \\
\ &\leq \ \left(2 \cdot 36^2 d \log(2n) \right)^{1/3} \left( \frac{\#_n(g)}{\gamma^2} \right)^{1/6} \left( \frac{1}{\#_n(g)} \right)^{1/2} \\
\ &\leq \ 14 \left( \frac{ d \log(2n) }{\gamma \#_n(g)} \right)^{1/3}. \qed
\end{align*}

\subsection{Proof of \pref{prop:neg-scoring-functions}}

Let $\Xcal = \{x_0, x_1, x_2 \}$ and $g_i = \{ x_0, x_i \}$ for $x_i=1,2$. Let the marginal distribution over $\Xcal$ be uniform. We will consider the 3 class classification setting, where $h_1(x) = 1$ and $h_2(x) = 2$ for all $x \in \Xcal$. We will consider two scenarios for the conditional distribution of $y$ given $x$.

In scenario 1, we have
\begin{align*}
P(y \mid x_0) & = \begin{cases}
1/4 & \text{ if } y=1 , \\
0 & \text{ if } y = 2 , \\
3/4 & \text{ if } y = 3 , \\
\end{cases} \\
P(y \mid x_1) & = \begin{cases}
3/4 & \text{ if } y=1 , \\
1/4 & \text{ if } y = 2 , \\
0 & \text{ if } y=3 ,
\end{cases} \\
P(y \mid x_2) & = \begin{cases}
1 & \text{ if } y=2 , \\
0 & \text{ if } y \in \{ 1,3 \} .
\end{cases}
\end{align*}
Abusing notation, we have under scenario 1:
\begin{align*}
L(h_1 \mid x_0) = 3/4 & & L(h_1 \mid x_1) = 1/4 & & L(h_1 \mid x_2) = 1 \\
L(h_2 \mid x_0)  = 1 & & L(h_2 \mid x_1) = 3/4 & & L(h_2 \mid x_2) = 0 \\
L(h_1 \mid g_1) = 1/2  & & L(h_1 \mid g_2) = 7/8 & &\\
L(h_2 \mid g_1) = 7/8 & & L(h_2 \mid g_2) = 1/2. & &
\end{align*}
Observe that under scenario 1, for any decision list $f \in \DL[\Gcal; \Hcal]$ that does not order $(h_1, g_1)$ at the beginning there exists $g \in \Gcal$ such that
\[ L(f \mid g) \ \geq \ \min_{h \in \Hcal} L(h \mid g) + 1/8. \]  

In scenario 2, we have
\begin{align*}
P(y \mid x_0) & = \begin{cases}
0 & \text{ if } y=1 , \\
1/4 & \text{ if } y = 2 , \\
3/4 & \text{ if } y = 3 , \\
\end{cases} \\
P(y \mid x_1) & = \begin{cases}
1 & \text{ if } y=1 , \\
0 & \text{ if } y \in \{ 2,3 \},
\end{cases} \\
P(y \mid x_2) & = \begin{cases}
1/4 &\text{ if } y = 1 , \\
3/4 & \text{ if } y=2 , \\
0 & \text{ if } y = 3 .
\end{cases}
\end{align*}
Under scenario 2:
\begin{align*}
L(h_1 \mid x_0)  = 1 & & L(h_1 \mid x_1) = 0 & & L(h_1 \mid x_2) = 3/4 \\
L(h_2 \mid x_0) = 3/4 & & L(h_2 \mid x_1) = 1 & & L(h_2 \mid x_2) = 1/4 \\
L(h_1 \mid g_1) = 1/2  & & L(h_1 \mid g_2) = 7/8 & &\\
L(h_2 \mid g_1) = 7/8 & & L(h_2 \mid g_2) = 1/2. & &
\end{align*}
Conversely, we have that under scenario 2, for any decision list $f \in \DL[\Gcal; \Hcal]$ that does not order $(h_2, g_2)$ at the beginning there exists $g \in \Gcal$ such that
\[ L(f \mid g) \ \geq \ \min_{h \in \Hcal} L(h \mid g) + 1/8. \]

\section{Missing proofs from \pref{sec:sleeping-experts}}
\label{app:sleeping-experts}

\subsection{\mlchedge algorithm and guarantees}

\begin{algorithm}[t]
\caption{\mlchedge in the multi-group setting}
\label{alg:mlchedge}
\begin{algorithmic}
\REQUIRE Groups $\Gcal$, hypothesis class $\Hcal$, learning rates $\eta_{h,g} \in [0,1]$.
\ENSURE Internal hypotheses $p_1(\cdot; \cdot), \ldots p_n( \cdot; \cdot)$ .
\STATE Initialize weights $w^{(0)}_{h,g} = \frac{1}{|\Hcal| |\Gcal|}$.
\FOR{$t=1,2,\ldots, n$}
	\STATE Define 
	\[ p_t((h,g); x) \ := \ \frac{g(x)(1- e^{-\eta_{h,g}}) w^{(t-1)}_{h,g}}{\sum_{h',g'} g'(x)(1- e^{-\eta_{h',g'}}) w^{(t-1)}_{h',g'}}. \]
	\STATE Receive point $(x_t, y_t)$ and incur loss 
	\[ \hat{\ell}_t \ = \ \sum_{h,g} g(x_t) \ell(h(x_t), y_t) p_t((h,g); x_t).  \]
	\STATE Update weight vectors
	\[ w^{(t)}_{h,g} \ = \ w^{(t-1)}_{h,g}  \exp\left( \eta_{h,g} g(x_t) \left( \hat{\ell}_t e^{-\eta_{h,g}} - \ell(h(x_t), y_t) \right) \right) .  \]
\ENDFOR
\OUTPUT $p_1, \ldots, p_n$.
\end{algorithmic}
\end{algorithm}

\pref{alg:mlchedge} displays \mlchedge, as presented by \citet{gaillard2014second}, in the multi-group learning setting. 
Theorem~16 of \citet{gaillard2014second} translates as follows.

\begin{theorem}
\label{thm:MLC-hedge-guarantee}
Let $\eta_{h,g} \in [0,1]$ be the learning rate assigned to expert $(h,g)$, and suppose that the initial probabilities are uniform over the experts. For each expert $(h,g)$, the cumulative loss of \mlchedge satisfies
\[ \sum_{t=1}^n g(x_t) (\hat{\ell}_t - \ell(h(x_t), y_t)) \ \leq \ \left( e - 1 + \frac{1}{\eta_{h,g}} \right) \log (|\Hcal||\Gcal|) + (e-1) \eta_{h,g} \sum_{t=1}^n  g(x_t) \ell(h(x_t), y_t). \]
\end{theorem}

\subsection{An online-to-batch guarantee}
For a collection of internal hypotheses $p_1, \ldots, p_n$ and a distribution $Q$ over such hypotheses, we use the notational conventions
\begin{align*}
L(p_t \mid g) \ &:= \ \EE_{(x,y)} \left[ \EE_{(\tilde{h}, \tilde{g}) \sim p_t( \cdot \,  ; \, x)} \left[ \ell(\tilde{h}(x), y) \right] \mid g \right] \\
L(Q \mid g) \ &:= \ \EE_{p_t \sim Q} \left[ L(p_t \mid g) \right].
\end{align*}

The following lemma shows that the average population losses of these internal hypotheses can be bounded in terms of their average empirical performance.
\begin{lemma}
\label{lem:online-to-batch}
Suppose the loss function is bounded in the range $[0,1]$. Let $p_1, \ldots, p_n$ be a sequence of hypotheses produced by an online learning algorithm on an i.i.d.\ sequence $(x_1, y_1), \ldots, (x_n, y_n)$ with associated losses $\hat{\ell}_1, \ldots, \hat{\ell}_n$. Then with probability at least $1-\delta$, we have for all $g \in \Gcal$ simultaneously
\begin{align*}
 \frac{1}{n} \sum_{t=1}^n L(p_t \mid g) 
\ \leq \ \frac{1}{n} \sum_{t=1}^n \frac{g(x_t)}{P(g)} \hat{\ell}_t + \sqrt{\frac{1}{n P(g)} \log \frac{|\Gcal|}{\delta}} + \frac{2}{3n P(g)} \log \frac{|\Gcal|}{\delta}.
\end{align*}
\end{lemma}

A key ingredient in the proof of \pref{lem:online-to-batch} is Freedman's inequality~\citep{Fre75}.
\begin{theorem}[Freedman's inequality]
\label{thm:Freedman}
Let $V_1, \ldots, V_T$ be a martingale difference sequence with respect to filtration $\Fcal_t$ such that there exist constants $a, b \geq 0$ satisfying
\begin{itemize}
	\item $|V_t| \leq a$ for all $t=1,\ldots, T$ with probability 1 and
	\item $\sum_{t=1}^T \EE[V_t^2 \mid \Fcal_{t-1}] \leq b^2$.
\end{itemize}
Then with probability at least $1-\delta$, we have
\[ \sum_{t=1}^T V_t \leq \frac{2}{3} a \log \frac{1}{\delta} + b \sqrt{2 \ln \frac{1}{\delta}}. \] 
\end{theorem} 

Our proof of \pref{lem:online-to-batch} is similar to the online-to-batch reduction of~\citet{CCG04}. Namely, fix $g \in \Gcal$ and define the random variable 
\[ V_t \ = \ \frac{1}{n} L(p_t \mid g) - \frac{1}{n P(g)} g(x_t) \hat{\ell}_t \ = \ \frac{1}{n} L(p_t \mid g) - \frac{1}{n P(g)} g(x_t) \sum_{h,g'} g'(x_t) p_t(h,g' ;  x_t) \ell(h(x_t), y_t). \]
Notice that $V_1, \ldots, V_n$ form a martingale difference sequence. Moreover, $V_t \in \left[-\frac{1}{n P(g)}, \frac{1}{n P(g)} \right]$. Letting $\Fcal_{t}$ denote the sigma-field of all outcomes up to time $t$, we can calculate
\begin{align*}
& \EE\left[ \left( g(x_t) \sum_{h,g'} g'(x_t) p_t(h,g' ;  x_t) \ell(h(x_t), y_t) \right)^2 \mid \Fcal_{t-1} \right] \\
\ & \hspace{10em} = \ P(g) \EE\left[ \EE\left[ \left( \sum_{h,g'} g'(x_t) p_t(h,g' ;  x_t) \ell(h(x_t), y_t) \right)^2 \mid g \right]\mid \Fcal_{t-1} \right] \\
\ &  \hspace{10em}  \leq \ P(g) \EE\left[ \EE\left[ \sum_{h,g'} g'(x_t) p_t(h,g' ;  x_t) \ell(h(x_t), y_t)^2 \mid g \right]\mid \Fcal_{t-1} \right] \\
\ &  \hspace{10em}  \leq \ P(g)  \EE\left[ L(p_t \mid g) \mid \Fcal_{t-1} \right] \ = \  P(g)  L(p_t \mid g)
\end{align*}
where the first inequality is Jensen's inequality and the second follows from the fact that the losses lie in $[0,1]$. Thus,
\begin{align*}
\EE[V_{t}^2 \mid \Fcal_{t-1}] \ \leq \ \frac{1}{n^2 P(g)} L(p_t \mid g) - \frac{1}{n^2 }  L(p_t \mid g)^2 \ \leq \ \frac{1}{n^2 P(g)} L(p_t \mid g).
\end{align*}
Freedman's inequality then implies that with probability at least $1-\delta/|\Gcal|$,
\begin{align*}
\frac{1}{n} \sum_{t=1}^n L(p_t \mid g) 
\ &\leq \ \frac{1}{n P(g)} \sum_{t=1}^T g(x_t) \hat{\ell}_t+ \frac{1}{n} \sqrt{ \frac{1}{P(g)} \sum_{t=1}^n L(p_t \mid g)  \log \frac{|\Gcal|}{\delta}} + \frac{2}{3n P(g)} \log \frac{|\Gcal|}{\delta} \\
\ &\leq \ \frac{1}{n P(g)} \sum_{t=1}^T g(x_t)\hat{\ell}_t + \sqrt{ \frac{1}{nP(g)} \log \frac{|\Gcal|}{\delta}} + \frac{2}{3n P(g)} \log \frac{|\Gcal|}{\delta},
\end{align*}
where we have again used the fact that the losses lie in $[0,1]$. Taking a union bound over $\Gcal$ finishes the proof. \qed

\subsection{Proof of \pref{thm:sleeping-experts}}
We will show that with probability at least $1-\delta$, the predictor $Q$ returned by \pref{alg:sleeping-experts} satisfies
\[ L(Q \mid g) \ \leq \ \min_{h \in \Hcal} L(h \mid g) + 60 \sqrt{\frac{D}{\#_n(g)}} + \frac{16D}{\#_n(g)} \qquad \forall g \in \Gcal, \]
where $D = 2 \log(|\Hcal| |\Gcal|) + \log \frac{64}{\delta}$

Let $m = \lfloor n/2 \rfloor$, and let $(x_1, y_1), \ldots, (x_{m}, y_m)$, $(x'_1, y'_1), \ldots, (x'_{m}, y'_m)$ be the data split utilized by \pref{alg:sleeping-experts}. For these two splits of our data, we will use the notation
\begin{align*}
S_m(g) \ &=  \ \sum_{i=1}^m g(x_i) \\
S_m'(g) \ &= \ \sum_{i=1}^m g(x_i') \\
L_m(h \mid g) \ &= \  \frac{1}{S_m(g)} \sum_{i=1}^m g(x_i) \ell(h(x_i),y_i).
\end{align*}
From \pref{thm:conv_conditional_losses} and \pref{lem:relative-deviations-0,1}, we have that with probability at least $1-\delta/2$,
\begin{align*}
 L_m(h \mid g) \ \leq \ L(h \mid g) + 9 \sqrt{\frac{D}{S_m(g)}} \\
 -2 \sqrt{S_m(g) D}  \ \leq \ m P(g) - S_m(g) \ \leq \ 2 \sqrt{S_m(g) D} + 4D \\
  -2 \sqrt{S'_m(g) D}  \ \leq \ m P(g) - S'_m(g) \ \leq \ 2 \sqrt{S'_m(g) D} + 4D \\
   -2 \sqrt{\#_n(g) D}  \ \leq \ n P(g) - \#_n(g) \ \leq \ 2 \sqrt{\#_n(g) D} + 4D
\end{align*}  
for all $h \in \Hcal$ and $g \in \Gcal$. Moreover, combining \pref{thm:MLC-hedge-guarantee} and \pref{lem:online-to-batch}, we have that with probability at least $1-\delta/2$,
\begin{align*}
L(Q \mid g) \ &= \ \frac{1}{m} \sum_{t=1}^m L(p_t \mid g) \\
\ &\leq \ \frac{1}{m} \sum_{t=1}^m \frac{g(x_t)}{P(g)} \hat{\ell}_t + \sqrt{\frac{1}{m P(g)} \log \frac{2|\Gcal|}{\delta}} + \frac{2}{3m P(g)} \log \frac{2|\Gcal|}{\delta}  \\
\ & \leq \ \frac{1}{m P(g)} \left[  \sum_{t=1}^m  g(x_t) \ell(h_g(x_t), y_t) +  \left( e - 1 + \frac{1}{\eta_{h_g,g}} \right) \log (|\Hcal||\Gcal|) + (e-1) \eta_{h_g,g} \sum_{t=1}^m  g(x_t) \ell(h(x_t), y_t) \right] \\
\ & \hspace{3em} + \sqrt{\frac{1}{m P(g)} \log \frac{2|\Gcal|}{\delta}} + \frac{2}{3m P(g)} \log \frac{2|\Gcal|}{\delta} \\
\ & \leq \ \frac{1}{m P(g)} \left[  \sum_{t=1}^m  g(x_t) \ell(h(x_t), y_t) +  \left( 2 + \frac{1}{\eta_{h_g,g}} \right) \log (|\Hcal||\Gcal|) + 2 \eta_{h,g} S_m(g) \right] \\
\ & \hspace{3em} + \sqrt{\frac{1}{m P(g)} \log \frac{2|\Gcal|}{\delta}} + \frac{2}{3m P(g)} \log \frac{2|\Gcal|}{\delta}
\end{align*}
for all $g \in \Gcal$, where $h_g := \argmin_{h \in \Hcal} L(h_g \mid g)$ and the last line has used the fact that the losses are restricted to $[0,1]$. By a union bound, with probability at least $1-\delta$ all of the above occurs. Let us condition on this happening.

Pick some $g \in \Gcal$. Observe that the theorem trivially holds if $\#_n(g) < 352D + 4$. Thus, we may assume $\#_n(g) \geq 352D + 4$. In this setting, we can then see that 
\begin{align*}
S_m(g) \ &\geq \ m P(g) -  2 \sqrt{S_m(g) D} + 4D \\
\ &\geq \ \left( \frac{n}{2} - 1 \right)P(g) -  2 \sqrt{S_m(g) D} - 4D \\
\ &\geq \ \frac{1}{2} \left( \#_n(g) - 2 \sqrt{\#_n(g) D} \right) -  2 \sqrt{S_m(g) D} - 4D - 1 \\
\ &\geq \ \frac{1}{2}\#_n(g)  - 3\sqrt{\#_n(g) D} - 4D -1 \\
\ &\geq \ \frac{1}{4} \#_n(g).
\end{align*}
Here the second-to-last line follows from the fact that $S_m(g) \leq \#_n(g)$ and the last line follows from our lower bound on $\#_n(g)$. By a similar chain of reasoning, we also have  $S'_m(g) \geq \frac{1}{4} \#_n(g)$. Given this, we can bound
\begin{align*}
\frac{1}{m P(g)} \sum_{t=1}^m  g(x_t) \ell(h(x_t), y_t) \ &\leq \ \frac{S_m(g)}{mP(g)} \left( L(h \mid g) + 9 \sqrt{\frac{D}{S_m(g)}} \right) \\
\ &\leq \ \frac{m P(g) + 2 \sqrt{S_m(g) D}}{m P(g)}  L(h \mid g) + \frac{9 \sqrt{D S_m(g)}}{m P(g)} \\
\ &\leq \  L(h \mid g) +  \frac{11\sqrt{D S_m(g)}}{S_m(g) - 2 \sqrt{D S_m(g)}} \ \leq \ L(h \mid g) + 22 \sqrt{\frac{D}{S_m(g)}}
\end{align*}
where the last inequality follows from the fact that $S_m(g) \geq \frac{1}{4}\#_n(g) > 16D$. Similarly, we also have
\begin{align*}
\frac{\sqrt{S'_m(g)}}{m P(g)} \ & \leq \ 2\sqrt{\frac{1}{S'_m(g)}}
\end{align*}
and
\begin{align*}
\frac{S_m(g)}{m P(g)} \ & \leq \ 2.
\end{align*}
Putting it all together, we have
\begin{align*}
L(Q \mid g) \ &\leq \ L(h \mid g) + 22 \sqrt{\frac{D}{S_m(g)}}  + \frac{2}{mP(g)} \left(  \log(|\Gcal||\Hcal|) + \frac{1}{3} \log \frac{2|\Gcal|}{\delta}\right)  \\
& \hspace{3em} + \frac{\sqrt{S'_m(g) \log (|\Gcal||\Hcal|)}}{m P(g)} + \frac{2 S_m(g)}{mP(g)} \sqrt{\frac{\log (|\Gcal||\Hcal|)}{S'_m(g)}} + \sqrt{\frac{1}{mP(g)} \log \frac{2|\Gcal|}{\delta}} \\
\ &\leq \ L(h \mid g) + 22 \sqrt{\frac{D}{S_m(g)}}  +  \frac{4D}{S_m(g)}  + 2 \sqrt{\frac{D}{S'_m(g)}} + 4 \sqrt{\frac{D}{S'_m(g)}} + 2 \sqrt{\frac{D}{S_m(g)}} \\
\ & \leq \  L(h \mid g)  + 60 \sqrt{\frac{D}{\#_n(g)}} + \frac{16 D}{\#_n(g)}. \qed
\end{align*}

\section{Insufficiency of multiaccuracy}
\label{app:insuf-multiaccuracy}

In the binary prediction setting where $\Ycal = \{ 0, 1 \}$, \citet{kim2019multiaccuracy} use the definition that a function $f:\Xcal \rightarrow [0,1]$ is $\alpha$-multiaccurate with respect to $\Ccal \subset[-1,+1]^\Xcal$ if 
\[ \EE_{x, y} \left[ c(x)(f(x) - y)  \right]  \ \leq \ \alpha \]
for all $c \in \Ccal$. To simplify the discussion, let us assume that the label $y$ is a deterministic function of the corresponding $x$, i.e. there exists $\eta \in \{0,1 \}^\Xcal$ such that $\EE[y \mid x] = \eta(x)$. In this setting, \citeauthor{kim2019multiaccuracy} showed that multiaccuracy can be translated into a notion of multi-group learning as follows.
\begin{proposition}[Proposition 1 from \citet{kim2019multiaccuracy}]
\label{prop:multiacc-kim}
Let $\Ccal$ and $S \subseteq \Xcal$ be given, and suppose that $f:\Xcal \rightarrow [0,1]$ is $\alpha$-multiaccurate with respect to $\Ccal$. Further define $\hat{\eta}(x) = 1 - 2 \eta(x)$. If there exists a $c \in \Ccal$ such that 
\[\EE_{x,y}[|c(x) - \hat{\eta}(x) \ind[x \in S]|] \leq \tau,\] then 
\[ \pr_{x,y}\left( \sgn(f(x) - 1/2) \neq \sgn(y-1/2) \mid x \in S \right) \ \leq \ \frac{2}{P(S)} \left( \alpha + \tau \right) .\]
\end{proposition}
In the multi-group learning setup where $\Hcal \subset [-1,+1]^\Xcal$ and $\Gcal \subset \{ 0, 1\}^\Xcal$, we can take $\Ccal = \{ x \mapsto h(x) g(x) \mid h \in \Hcal, g \in \Gcal \}$. \pref{prop:multiacc-kim} tells us that $\alpha$-multiaccuracy with respect to $\Ccal$ implies
\begin{align}
\label{eqn:multiacc-consequence}
\pr_{x,y}\left( \sgn(f(x) - 1/2) \neq \sgn(y-1/2) \mid x \in g \right) \ \leq \ 4 \inf_{h \in \Hcal} L(h \mid g) + \frac{2\alpha}{P(g)} \quad \text{for all $g \in \Gcal$} ,
\end{align}
where the loss in $L(\cdot \mid \cdot)$ is zero-one loss.

When each group has a corresponding low error classifier in $\Hcal$, Eq.~\eqref{eqn:multiacc-consequence} tells us that multiaccuracy leads to reasonably good predictions across all groups. However, the bound in Eq.~\eqref{eqn:multiacc-consequence} devolves to no better than random guessing whenever $\inf_{h \in \Hcal} L(h \mid g) \geq 1/8$. One may ask if this is due to some slack in the proof of \pref{prop:multiacc-kim} or if it is some intrinsic looseness associated with multiaccuracy. The following result shows that multiaccuracy reduction must result in at least some constant in front of $\inf_{h \in \Hcal} L(h \mid g)$.

\begin{proposition}
\label{prop:multiacc-lower-bound}
Suppose $|\Xcal| \geq 3$ and let $\epsilon > 0$. There exist $\eta, f, g \in \{0,1 \}^\Xcal$, $h \in \{ -1, +1 \}^\Xcal$, and a marginal distribution over $\Xcal$ such that 
\begin{itemize}
	\item $f$ is $0$-multiaccurate with respect to $\Ccal = \{ h \cdot g\}$,
	\item $\pr_{x} \left( h(x) \neq \eta(x) \mid x \in g \right) = \epsilon$, but
	\item $\pr_{x} \left( f(x) \neq \eta(x) \mid x \in g \right) \ = \ 2 \epsilon$.
\end{itemize}
\end{proposition}
\begin{proof}
Suppose $\Xcal = \{ x_0, x_1, x_2 \}$ with $P(x_0) = 1-2\epsilon$ and $P(x_1) = P(x_2) = \epsilon$. Let $\eta(x) = g(x) = 1$ for all $x \in \Xcal$ and define 
\begin{align*}
f(x) \ = \ \begin{cases}
0 & \text{ if } x \in \{ x_1, x_2 \} \\
1 & \text{ if } x = x_0
\end{cases} & & \text{ and } & &
h(x) \ = \ \begin{cases}
1 & \text{ if } x \in \{ x_0, x_1 \} \\
-1 & \text{ if } x = x_2
\end{cases} .
\end{align*}
Then we can establish the following facts.
\begin{enumerate}
	\item $\pr_{x} \left( h(x) \neq \eta(x) \mid x \in g \right)  = P(x_2) = \epsilon$.
	\item $\pr_{x} \left( f(x) \neq \eta(x) \mid x \in g \right) = 1-P(x_0) = 2 \epsilon$.
	\item For $c =  h\cdot g$, we have
	\[ \EE[ c(x)(f(x) - \eta(x))] \ = \ \EE[h(x) f(x)] - \EE[h(x)] \ = \ 1-2 \epsilon - (1-2\epsilon) \ = \ 0. \]
	Thus, $f$ is 0-multiaccurate with respect to $\Ccal = \{ c\}$.
\end{enumerate}
Combining all of the above gives the proposition.
\end{proof}

Thus, to get multi-group learning bounds of the form in Eq.~\eqref{eqn:hs_obj_abs} or Eq.~\eqref{eqn:hs_obj_rel}, we must go beyond this type of simple application of multiaccuracy.